\documentclass{article}

\usepackage{microtype}
\usepackage{graphicx}
\usepackage{subcaption} % defines subfigure environment
\usepackage{booktabs} % for professional tables

\usepackage{arxiv}

\usepackage[utf8]{inputenc} % allow utf-8 input
\usepackage[T1]{fontenc}    % use 8-bit T1 fonts
\usepackage{hyperref}       % hyperlinks
\usepackage{url}            % simple URL typesetting
\usepackage{booktabs}       % professional-quality tables
\usepackage{amsfonts}       % blackboard math symbols
\usepackage{nicefrac}       % compact symbols for 1/2, etc.
\usepackage{microtype}      % microtypography
\usepackage{lipsum}		% Can be removed after putting your text content
\usepackage{graphicx}
\usepackage{natbib}
\usepackage{doi}

\usepackage{amsmath}
\usepackage{wrapfig}
\usepackage{amssymb}
\usepackage{mathtools}
\usepackage{amsthm}
\usepackage{xcolor}
\usepackage{algorithm}
\usepackage{algorithmic}
\usepackage{amsmath,amssymb}
\usepackage{caption}
\usepackage{aliascnt}
\usepackage[capitalize,noabbrev]{cleveref}

\theoremstyle{plain}
\newtheorem{theorem}{Theorem}[section]
\newaliascnt{proposition}{theorem}
\newtheorem{proposition}[proposition]{Proposition}
\aliascntresetthe{proposition}
\newaliascnt{challenge}{theorem}
\newtheorem{challenge}[challenge]{Challenge}
\aliascntresetthe{challenge}
\newaliascnt{question}{theorem}
\newtheorem{question}[question]{Question}
\aliascntresetthe{question}

\theoremstyle{definition}

\theoremstyle{remark}

\title{A Long-Short Flow-Map Perspective for Drifting Models}

\author{ \hspace{1mm}Zhiqi Li\\
	School of Interactive Computing\\
	Georgia Institute of Technology\\
	\texttt{zli3167@gatech.edu} \\
	\And
	Bo Zhu \\
	School of Interactive Computing\\
	Georgia Institute of Technology\\
	\texttt{bo.zhu@gatech.edu} \\
}

\renewcommand{\shorttitle}{A Long-Short Flow-Map Perspective for Drifting Models}

\hypersetup{
pdftitle={A Long-Short Flow-Map Perspective for Drifting Models},
pdfauthor={Zhiqi Li, Bo Zhu},
pdfkeywords={One-Step Generation, Flow Map Method, Drifting Model, Closed-Form Solution},
}

\begin{document}
\maketitle

\begin{abstract}
	This paper provides a reinterpretation of the Drifting Model~\cite{deng2026generative} through a semigroup-consistent long-short flow-map factorization. We show that a global transport process can be decomposed into a long-horizon flow map followed by a short-time terminal flow map admitting a closed-form optimal velocity representation, and that taking the terminal interval length to zero recovers exactly the drifting field together with a conservative impulse term required for flow-map consistency. Based on this perspective, we propose a new likelihood learning formulation that aligns the long-short flow-map decomposition with density evolution under transport. We validate the framework through both theoretical analysis and empirical evaluations on benchmark tests, and further provide a theoretical interpretation of the feature-space optimization while highlighting several open problems for future study.
\end{abstract}

% keywords can be removed

\newcommand{\bo}[1]{\textcolor{blue}{[Bo: #1]}}
\newcommand{\zhiqi}[2]{\textcolor{blue}{[Zhiqi: #2]}}

\section{Introduction}

Flow-based generative modeling formulations such as Flow Matching~\cite{lipman2023flow, lipman2024flow} describe the evolution of a probability distribution through a time-dependent velocity field that transports a simple base distribution to the data distribution. While these continuous-time methods provide stable regression objectives and strong empirical performance, sample generation typically requires numerically integrating the learned dynamics over tens to hundreds of steps. This sampling cost has motivated the development of few-step and one-step generative frameworks~\cite{frans2025shortcut, geng2025mean, geng2025improved, song2023consistency, zhou2025inductive}, which aim to learn direct transport maps that produce samples in a single evaluation.

Among these approaches, the Drifting Model~\cite{deng2026generative} was proposed as a deterministic one-step generator motivated from a distribution-matching perspective and has demonstrated competitive results. The method learns a deterministic mapping $f$ such that the pushforward of the noise distribution by $f$ matches the data distribution, enabling one-step sampling by $x_1=f(x_0)$. While the method is well justified from its original formulation, its connection to established transport and flow-map frameworks has not yet been explicitly characterized. In particular, it is of interest to understand how a one-step transport map relates to trajectory-consistent flow-map structures, and how dataset-level supervision arises within such a formulation without multi-step distillation or optimization.

%\begin{figure}[t]
%  \centering
%  \includegraphics[width=0.5\columnwidth]{images/drifting_main2.png}
%  \caption{\textbf{Overview of the Long-Short Flow Map framework.}
%  Leveraging flow-map trajectory consistency, we decompose the full map $\psi_{0\to 1}$ into a long map $\psi_{0\to 1-\Delta t}$ and a short map $\psi_{1-\Delta t\to 1}$. The short map is approximated using forward Euler or the trapezoidal rule and then computed via the closed-form solution in flow matching, providing dataset-level supervision for learning the long map $\psi_{0\to 1-\Delta t}$. Taking the limit $\Delta t \to 0$ yields a long-short flow-map derivation of the Drifting Model.}
%  \label{fig:teaser}
%\end{figure}

In this paper, we provide a flow-map-based interpretation and derivation of the Drifting Model. Our key message is that the Drifting Model can be derived from a semigroup-consistent long-short flow-map decomposition coupled with a terminal closed-form optimal velocity. Specifically, we factorize the global transport map according to the semigroup property of flow maps into a long-range, initial flow map and a short-range, terminal flow map, $\psi_{0\to1}=\psi_{1-\Delta t\to1}\circ\psi_{0\to1-\Delta t}$, where the terminal flow map is evaluated using the closed-form optimal velocity \cite{bertrand2025closed} near the final time. This construction yields direct dataset-level supervision for the terminal correction while preserving trajectory consistency for the remaining component. By analyzing the limit as $\Delta t \to 0$, we show that the resulting infinitesimal terminal map recovers exactly the drifting formulation, with the attraction term arising from the first-order expansion of the closed-form velocity and the impulse term appearing in the second-order expansion as a conservative correction required for flow-map consistency.

This long-short flow-map perspective offers a structured way to understand the Drifting Model and helps interpret several of its design choices, including the simultaneous presence of attraction and impulse components and the emergence of multiple feature representations. Building on this viewpoint, we further introduce a likelihood-learning formulation aligned with the induced density evolution under transport and propose a squared-kernel modification to improve theoretical consistency with the noise's form. We validate the resulting framework on both illustrative and benchmark tests and conclude with open questions in feature-space optimization.

\section{Background}

\begin{figure}[t]
    \centering
    \begin{subfigure}[t]{0.30\textwidth}
        \centering
        \includegraphics[width=\linewidth, trim=6pt 6pt 6pt 6pt, clip]{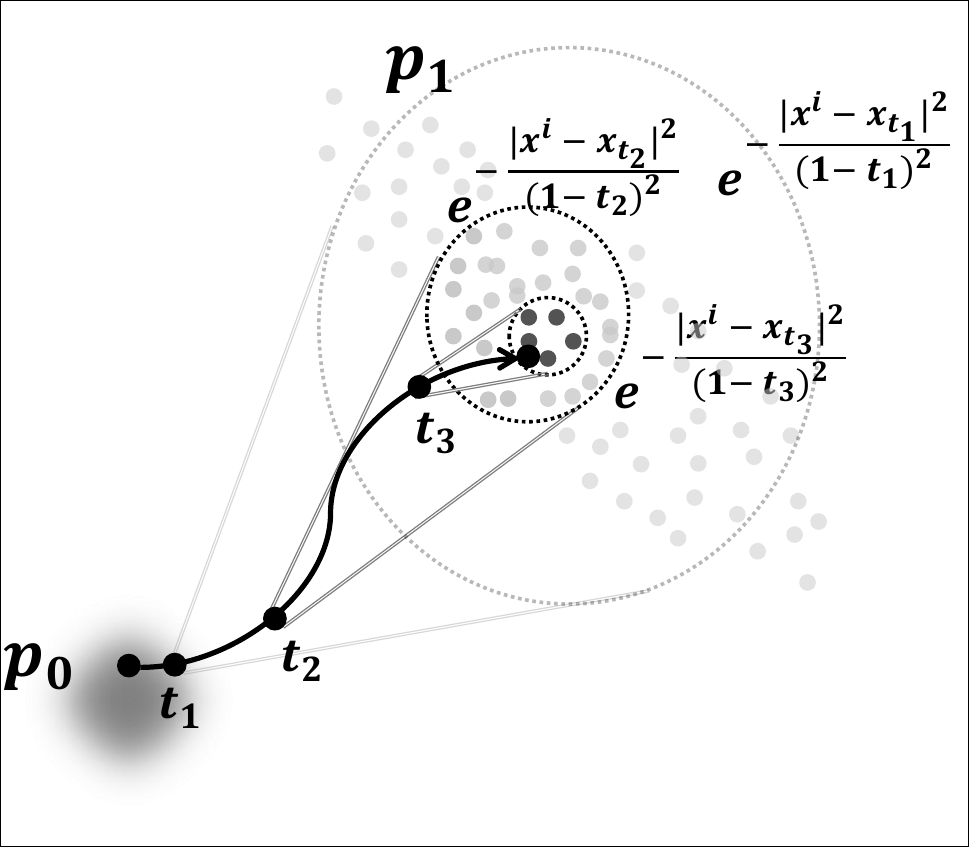}
        \caption{Closed-Form Flow Matching}
        \label{fig:b}
    \end{subfigure}\hfill
    \begin{subfigure}[t]{0.30\textwidth}
        \centering
        \includegraphics[width=\linewidth, trim=6pt 6pt 6pt 6pt, clip]{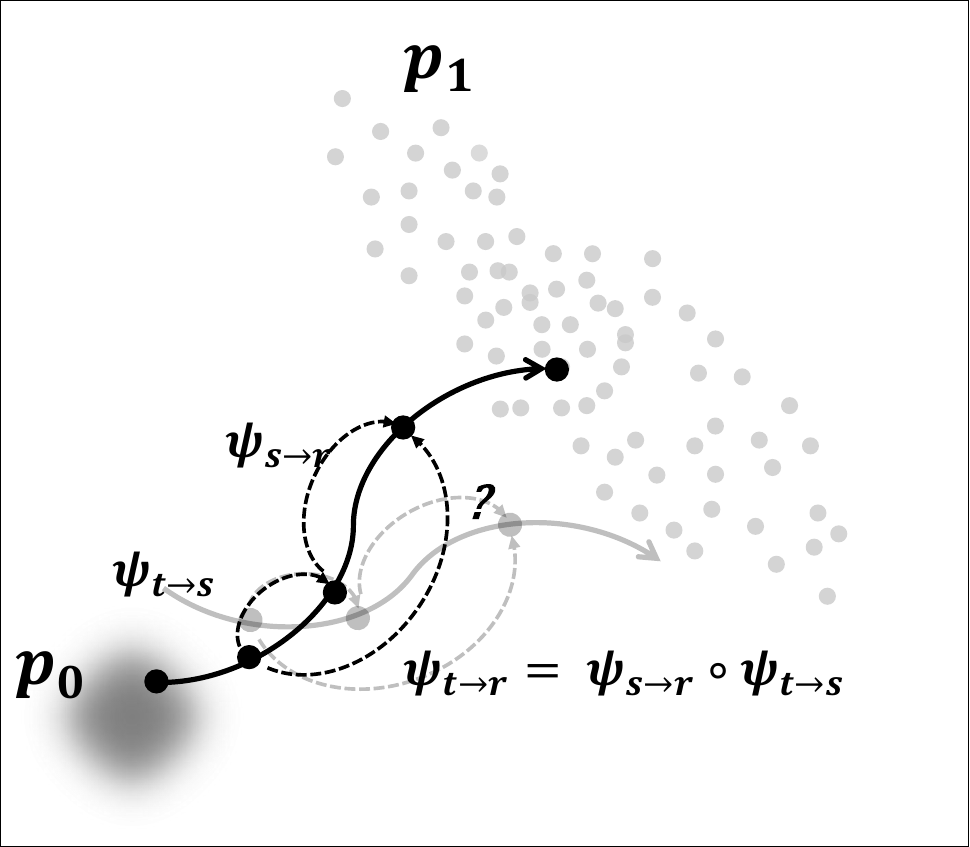}
        \caption{Flow-Map Method}
        \label{fig:a}
    \end{subfigure}\hfill
    \begin{subfigure}[t]{0.3\textwidth}
        \centering
        \includegraphics[width=\linewidth, trim=6pt 6pt 6pt 6pt, clip]{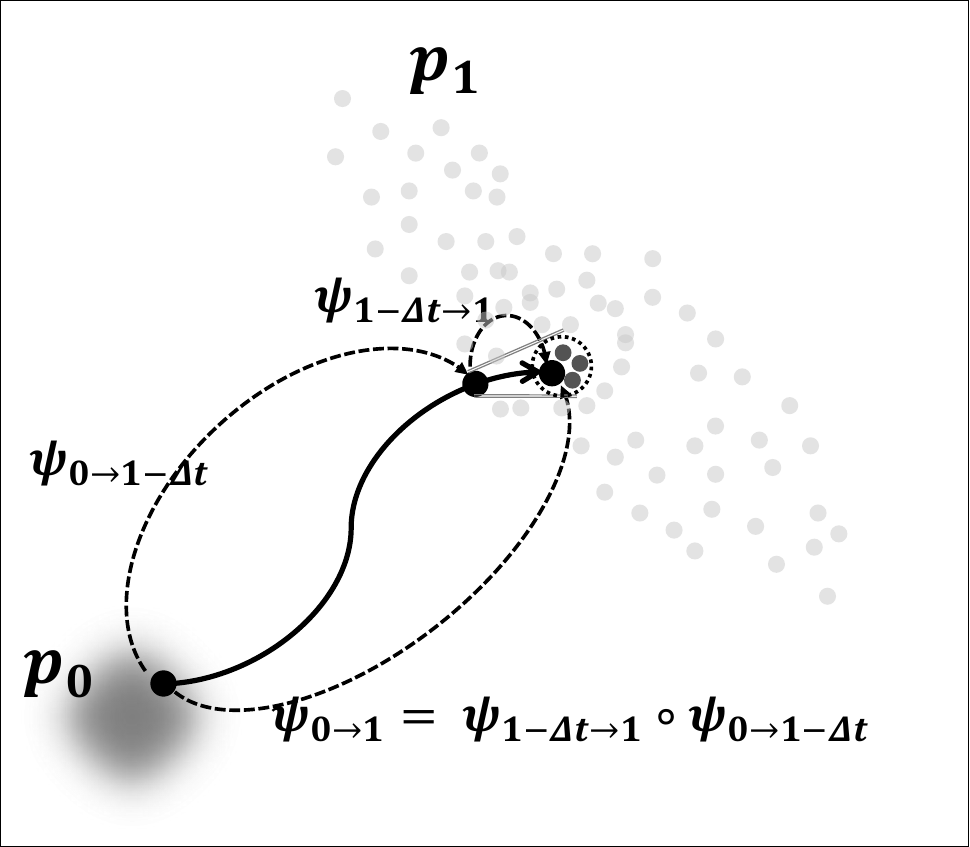}
        \caption{Long-Short Flow-Map Method}
        \label{fig:c}
    \end{subfigure}
    \caption{(a) Closed-form Flow Matching (\autoref{sec:close_form}) requires kernel-weighted aggregation over data points, which becomes prohibitively expensive when the kernel is diffuse (far from $t=1$). (b) Flow-map training (\autoref{sec:flow_map}) enforces trajectory consistency but is difficult to supervise directly from data, so spurious consistent paths (e.g., the semi-transparent segment) can still satisfy the constraint. (c) Our Long-Short Flow Map method (\autoref{sec:long_short_flow_map}) applies a closed form estimator on the short step $\psi_{1-\Delta t\to 1}$ as $\Delta t\to 0$ (\autoref{sec:first_order} first order, \autoref{sec:second_order} second order), which provides dataset grounded supervision for learning the long step $\psi_{0\to 1-\Delta t}$ and thereby addresses the challenges of both closed form Flow Matching and flow map methods.}
    \label{fig:three}
    \vspace{-1.5ex}
\end{figure}

\subsection{Closed-Form Flow Matching}\label{sec:close_form}
Given a dataset $\mathcal{D} = \{ x^i \in \mathcal{X}\}_{i=1}^n$ drawn from an unknown data distribution $p_{\mathrm{data}}$ on space $\mathcal{X}\subset \mathbb{R}^d$, Flow Matching aims to learn a continuous-time velocity field $u(x,t)$, $t \in [0,1]$, that transports a base distribution $p_0$, typically a Gaussian distribution $\mathcal{N}(0,\sigma^2)$, to the target distribution $p_1 = p_{\mathrm{data}}$ along a continuous path of intermediate distributions $(p_t)_{t\in[0,1]}$.  Once $u(x,t)$ is learned, samples from $p_1$ can be generated by integrating the ordinary differential equation
\begin{equation}\label{eq:ODE_sampling}
\frac{d x_t}{d t} = u(x_t, t),
\quad x_0 \sim p_0 .
\end{equation}
The family of distributions $(p_t)_{t\in[0,1]}$ is induced by the velocity field $u(x,t)$, and their relationship is governed by the continuity equation
\begin{equation}
    \begin{aligned}
    \frac{\partial}{\partial t} p_t(x) + \nabla \cdot \big( p_t(x)u_t(x) \big) = 0
    \end{aligned}
\end{equation}

A natural training objective for training $u_t^\theta(x)$ is $\mathcal{L}^{FM}(\theta) = \mathbb{E}_{t, x\sim p_t(x)}\|u_t^\theta(x) - u_t(x)\|^2$, which directly matches the model velocity to the reference velocity field.  However, this objective cannot be optimized in practice, since both $u_t(x)$ and the marginal distribution $p_t(x)$ are hard to be calculated directly from the dataset.  To incorporate supervision from data, Flow Matching introduces conditional velocities $u_t(x|x_1) = \frac{x_1-x}{1-t}$ and conditional flows $\psi_t(x| x_1)=t(x_1-x)+x$ for arbitrary $x_1 \in \mathcal{X}$. These conditional quantities induce a conditional distribution  $p_{t|1}(\cdot | x_1)=(\psi_t(\cdot| x_1))_\sharp p_0$, and the marginal velocity field and distribution can then be recovered by marginalization $u_t(x) = \mathbb{E}_{x_1\sim p_{1|t}(x_1|x)}[u_t(x | x_1)]$ and $p_t(x) = \mathbb{E}_{x_1\sim p_{1|t}(x_1|x)}[p_{t|1}(x | x_1)]$respectively.  Based on these constructions, Flow Matching defines the conditional surrogate objective 
\begin{equation}\label{eq:FlowMatchingLoss}
    \begin{aligned}
        \mathcal{L}_c^{FM}(\theta) = \mathbb{E}_{t,x_1\sim p_{data}, x\sim p_{t|1}(x|x_1)}\|u_t^\theta(x) - u_t(x|x_1)\|^2,        
    \end{aligned}
\end{equation}
which admits supervision from data samples.  It has been shown that $\nabla_\theta \mathcal{L}_c^{FM}(\theta) = \nabla_\theta \mathcal{L}^{FM}(\theta)$
and therefore $\mathcal{L}_c^{\mathrm{FM}}$ serves as a valid surrogate for optimizing $\mathcal{L}^{\mathrm{FM}}$.

In \cite{bertrand2025closed} and earlier work, the authors show that the Flow Matching objective in \autoref{eq:FlowMatchingLoss} admits a closed-form optimal solution, which can be summarized as the following property.
\begin{theorem}[Closed-Form Solution of Flow Matching \cite{bertrand2025closed}]\label{prop:close_form1}
\autoref{eq:FlowMatchingLoss} admits an optimal solution
\begin{equation}
u_t^*(x)
= \mathbb{E}_{x_1\sim p_{1|t}(x_1|x)}\!\big[u_t(x| x_1)\big]
= \frac{\mathbb{E}_{x_1\sim p_{1|t}(x_1|x)}[x_1]-x}{1-t},
\end{equation}
which further admits a closed-form expression (see \autoref{sec:proof_close_form1} for proof.)
\begin{equation}\label{eq:close_form1}
    \begin{aligned}
        u_t^*(x) &= \frac{\mathbb{E}_{x_1\sim p_1}[\frac{x_1-x}{1-t}k_{t}(x_1,x)]}{\mathbb{E}_{x_1\sim p_1}[k_{t}(x_1,x)]}, \qquad k_{t}(y,x) &= e^{-\frac{\|ty-x\|^2}{2(1-t)^2}}
    \end{aligned}
\end{equation}

\end{theorem}
In practice, this result is mainly used for analysis rather than directly computing the Flow Matching solution. We argue that the main obstacle to using the closed form as a solver arises away from $t=1$, leading to the following challenge.

%\begin{challenge}[Challenge for Closed-Form Computation]
%Directly using the closed-form optimal velocity $u_t^*(x)$ is impractical for training and sampling, due to the time-dependent concentration of the kernel $k_t$. When $t$ is close to $1$, the factor $(1-t)^{-2}$ in the exponent makes $k_t$ sharply concentrated, so the expectation is dominated by a small neighborhood of samples $x_1$ near $x$, and the computation is effectively local. In contrast, when $t$ is far from $1$, the kernel becomes much less concentrated, and accurately evaluating the closed form requires aggregating contributions from a large portion of the data distribution, which is computationally prohibitive. Consequently, the closed-form velocity is only tractable near $t=1$ and becomes impractical at earlier timesteps.
%\end{challenge}
\begin{challenge}[Challenge for Closed-Form Computation]
Directly using the closed-form optimal velocity $u_t^*(x)$ is impractical for training and sampling due to time-dependent concentration of kernel $k_t$. When $t$ is close to $1$, the factor $(1-t)^{-2}$ in the exponent makes $k_t$ sharply concentrated, so the expectation is dominated by samples $x_1$ near $x$ and the computation is effectively local. In contrast, when $t$ is far from $1$, the kernel is less concentrated, and evaluating the closed form requires aggregating contributions from a large portion of the data distribution, which is computationally prohibitive. Consequently, the closed-form velocity is tractable only near $t=1$ and impractical at earlier timesteps.
\end{challenge}

\subsection{One-Step Flow Map}\label{sec:flow_map}
Flow Matching suffers from slow generation because sampling with $u(x_t,t)$ amounts to numerically discretizing the evolution in \autoref{eq:ODE_sampling}, which typically requires many steps. Recent one-step generation methods instead learn a {flow map}, namely a family of mappings $\psi_{t\to r}:\mathcal{X}\to\mathcal{X}$ such that $\psi_{0\to t}(x_0)=x_t$ for any trajectory $\{x_t\}_{t=0}^1$ that solves the ODE initialized at $x_0$, and $\psi_{t\to r}=\psi_{0\to r}\circ \psi_{0\to t}^{-1}$ for any $t,r\in[0,1]$. The flow map satisfies
\begin{equation}\label{eq:evolv_flow_map}
\frac{\partial}{\partial t}\psi_{0\to t} = u_t\circ \psi_{0\to t},\qquad \psi_{0\to 0}=\mathrm{Id}_{\mathcal{X}}.
\end{equation}
The marginal path can be written as a pushforward $p_t = (\psi_{0\to t})_\sharp p_0$. Once $\psi_{t\to r}$ is learned, flow-map methods enable one-step generation via $x_1=\psi_{0\to 1}(x_0)$, and also few-step generation, e.g., $x_{0.5}=\psi_{0\to 0.5}(x_0)$ followed by $x_1=\psi_{0.5\to 1}(x_{0.5})$ for two steps.

Flow maps are characterized by a {trajectory consistency} property, also known as the {semigroup property}, which is the key requirement when learning $\psi_{t\to r}$. 
\begin{proposition}[Trajectory Consistency]\label{prop:trajector_consistency}
For any $0\le t \le s \le r \le 1$, the flow map satisfies
\begin{equation}
\psi_{t\to r} = \psi_{s\to r}\circ \psi_{t\to s}.
\end{equation}
\end{proposition}

To enforce trajectory consistency in learning, one can impose the following trajectory-consistency loss
\begin{equation}\label{eq:consistency_flow}
\mathcal{L}^{TC}(\theta)
=\mathbb{E}_{t,s,r,x_t\sim p_{t|1}(x|x_1), x_1\sim p_1}\Big[\|
\psi^\theta_{t\to r}(x_t)
-\psi^\theta_{s\to r}\!\big(\psi^\theta_{t\to s}(x_t)\big)
\|_2^2
\Big],
\end{equation}
%where $x_t=(1-t)x_0+tx_1\sim p_{t|1}(x_t|x_1)$, $x_1\sim p_{\text{data}}$, and $x_0\sim p_0$.
A central difficulty of optimizing \autoref{eq:consistency_flow} is the lack of supervision from the dataset. Unlike the velocity field $u_t(x)$ in Flow Matching, flow maps do not admit an easily computable "conditional flow map" from data pairs (see \autoref{thm:non_existence_conditioanl}), which prevents us from writing a simple conditional objective analogous to the standard conditional Flow Matching loss. 

To mitigate this issue, prior work has explored two broad strategies. \emph{Progressive extension} methods, exemplified by Shortcut \cite{frans2025shortcut}, SplitMeanFlow \cite{guo2025splitmeanflow}, and Flow Map Matching \cite{boffi2025flow}, start from the instantaneous limit $\lim_{s\to t}\frac{\psi_{t\to s}(x)-x}{s-t}=u_t(x)$ and progressively extend short-horizon supervision to longer horizons using the semigroup constraint in \autoref{prop:trajector_consistency}. \emph{Continuous-based methods} take the $s\to t$ limit of \autoref{prop:trajector_consistency} and derive a continuous equation of the form $\partial_t \psi_{t\to r}(x) + (u_t\cdot \nabla)\psi_{t\to r}(x)=0$, which brings the instantaneous velocity $u_t$ into the learning of long-horizon maps $\psi_{t\to r}$; since $u_t$ admits dataset supervision through its conditional form $u_t(x|x_1)$ in Flow Matching, this provides a way to inject data-driven supervision into flow-map learning, as exemplified by MeanFlow \cite{geng2025mean, geng2025improved}, which learns an averaged velocity field $u_{t\to r}(x) = \frac{\psi_{t\to r}(x) -x}{r-t}$ to parameterize the flow map. Despite these efforts, introducing reliable dataset-level supervision for long-range flow maps remains challenging in practice, as summarized below.

\begin{challenge}[Challenge for Flow-Map Methods]
A key challenge in flow-map learning is how to introduce {dataset-level supervision} for {long-range} maps $\psi_{t\to r}$ with large $r-t$. Progressive-extension methods do not directly use dataset supervision for long-horizon maps, and instead rely on indirect signals propagated from short-horizon maps through the semigroup constraint. Continuous-based methods inject supervision via the instantaneous conditional velocity in the limiting continuous equation, but they typically require differentiating through the map, which may hurt training stability and substantially increase computational cost.
\end{challenge}

\section{Long-Short Flow-Map Perspective}\label{sec:long_short_flow_map}

\begin{wrapfigure}{l}{0.5\columnwidth}
  \centering
  %\vspace{-0.8em}
  \includegraphics[width=1.0\linewidth, trim=6pt 6pt 6pt 6pt, clip]{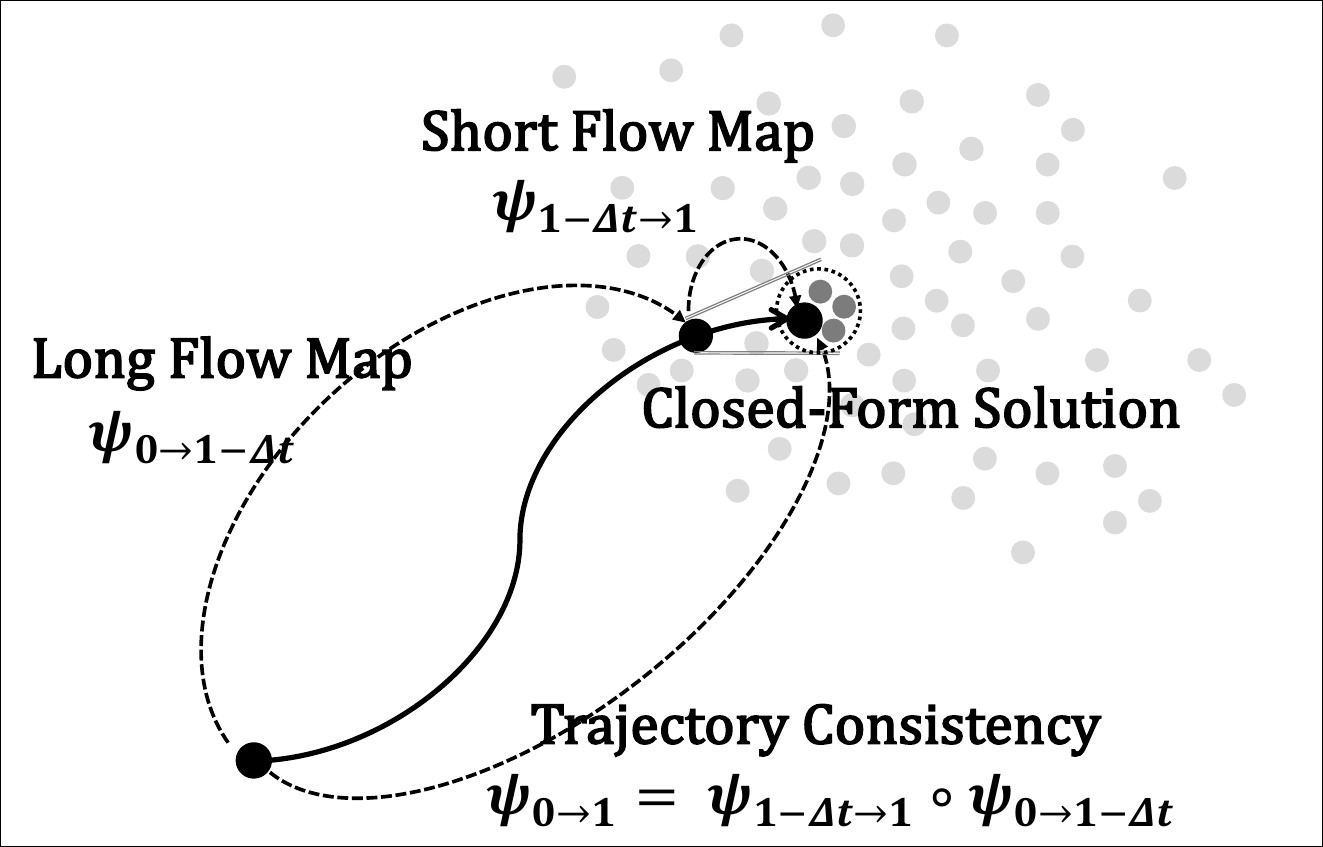}
  \caption{\textbf{Overview of the Long-Short Flow Map framework.}
  Leveraging flow-map trajectory consistency, we decompose the full map $\psi_{0\to 1}$ into a long map $\psi_{0\to 1-\Delta t}$ and a short map $\psi_{1-\Delta t\to 1}$. The short map is approximated using forward Euler or the trapezoidal rule and then computed via the closed-form solution in flow matching, providing dataset-level supervision for learning the long map $\psi_{0\to 1-\Delta t}$. Taking the limit $\Delta t \to 0$ yields a long-short flow-map derivation of the Drifting Model.}
  \label{fig:teaser}
  %\vspace{-1.0em}
\end{wrapfigure}

As discussed above, both the closed-form Flow Matching and flow-map approaches face practical challenges. The closed-form optimal velocity becomes impractical to evaluate when $t$ is far from the terminal time $1$, while flow-map learning struggles to obtain a dataset-level supervision when enforcing trajectory consistency. 
%Our key insight is that these two approaches can be combined via a long-short decomposition of flow maps, so that each mitigates the other's limitations. As in flow-map methods, we aim to learn $\{\psi_{t\to r}\}$. Using the trajectory-consistency property in \autoref{prop:trajector_consistency}, we start from the composition identity $\psi_{0\to 1} = \psi_{1-\Delta t \to 1}\circ \psi_{0\to 1-\Delta t}$. We interpret $\psi_{1-\Delta t\to 1}$ as a short flow map (a small step near the endpoint) and $\psi_{0\to 1-\Delta t}$ as a long flow map (a large transport from noise toward data), and refer to this factorization as the \textbf{long-short flow-map decomposition} (as shown in \autoref{fig:three}).  
Our key insight is that these two approaches can be combined via a long-short decomposition of flow maps, so that each mitigates the other's limitations. As in flow-map methods, we aim to learn ${\psi_{t\to r}}$. Using the trajectory-consistency property in \autoref{prop:trajector_consistency}, we start from the composition identity $\psi_{0\to 1} = \psi_{1-\Delta t \to 1}\circ \psi_{0\to 1-\Delta t}$. We use the short flow map $\psi_{1-\Delta t\to 1}$ for a small terminal step near the endpoint, and the long flow map $\psi_{0\to 1-\Delta t}$ for the main large-step transport from noise toward data. We refer to this factorization as the \textbf{long-short flow-map decomposition} (as shown in \autoref{fig:teaser}). The main idea is to use the flow-matching closed-form solution to approximate the short terminal map $\psi_{1-\Delta t\to 1}$, where the closed-form solution is accurate and computationally tractable, thereby providing a dataset-tied supervisory signal for learning the remaining long map $\psi_{0\to 1-\Delta t}$ through trajectory consistency. This coupling simultaneously alleviates the two challenges we mentioned above, effectively unifying the one-step flow map with the closed-form flow matching. We detail the derivation below.

%At a high level, the Drifting Model and Flow Map methods seem to pursue different objectives: the former performs distribution matching, whereas the latter directly learns a flow-map function. In this section, we show that the Drifting Model can be interpreted as an organic combination of the Flow Map framework and the close-form solution of Flow Matching, and we provide a derivation from the Flow Map perspective.

%As in Flow Map, we aim to learn the family of maps $\psi_{t\to r}$.  Leveraging the trajectory consistency property in \autoref{prop:trajector_consistency}, we consider the composition identity $\psi_{0\to 1} = \psi_{1-\Delta t \to 1}\circ \psi_{0\to 1-\Delta t}$.  As discussed earlier, a key challenge of closed-form solution of Flow Matching is that it is impractical when $t$ is far from $1$, whereas Flow Map learning suffers from the lack of direct supervision from the target dataset. Our main idea is to use the Flow Matching close-form solution to instantiate $\psi_{0\to 1-\Delta t}$, thereby injecting dataset supervision into the map learning. This simultaneously alleviates both issues and naturally yields the Drifting Model formulation, effectively unifying Flow Map learning with the Flow Matching closed form. We detail the derivation below.

\subsection{First-Order Approximation}\label{sec:first_order}

\begin{wrapfigure}{r}{0.4\columnwidth}
  \centering
  \vspace{-1.8em}
  \includegraphics[width=1.0\linewidth, trim=2pt 4pt 2pt 4pt, clip]{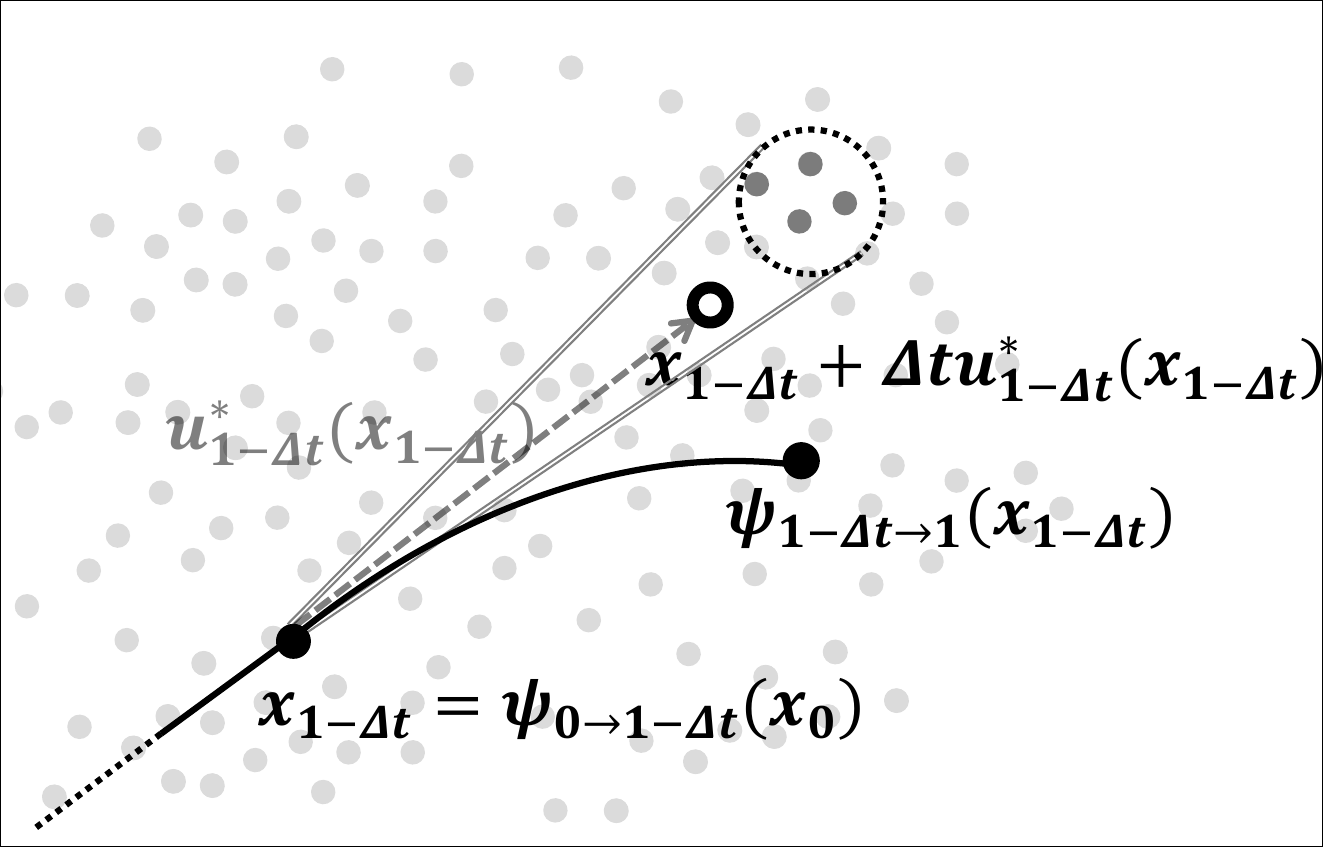}
  %\caption{\textbf{First-Order Approximation}}
  %\label{fig:first}
  \vspace{-1.0em}
\end{wrapfigure}

For $\psi^{\theta}_{0\to 1}(x_0) = \psi^{\theta}_{1-\Delta t\to 1}(\psi^{\theta}_{0\to 1-\Delta t}(x_0))$, let $x^\theta_{1-\Delta t}=\psi^{\theta}_{0\to 1-\Delta t}(x_0)$.  Using a \textbf{first-order forward Euler approximation} for the short interval $[1-\Delta t,1]$ via the instantaneous velocity field $u_{t}(\cdot)$, we obtain $\psi^{\theta}_{1-\Delta t\to 1}(x^\theta_{1-\Delta t})
        \approx
        x^\theta_{1-\Delta t} + \Delta tu_{1-\Delta t}(x^\theta_{1-\Delta t})$.  Substituting this approximation back into the consistency equation yields
\begin{equation}\label{eq:forward_euler_approx}
    \begin{aligned}
\psi^{\theta}_{0\to 1}(x_0)
\approx
\psi^{\theta}_{0\to 1-\Delta t}(x_0) + \Delta t u_{1-\Delta t}(x^\theta_{1-\Delta t}).        
    \end{aligned}
\end{equation}
Therefore, to learn the terminal flow map $\psi^{\theta}_{0\to 1}(x_0)$, we can consider an objective that enforces the above relation
\begin{equation}
    \begin{aligned}
\mathcal{L}(\theta) = \mathbb{E}_{x_0\sim p_0}\Big[
\|\psi^{\theta}_{0\to 1}(x_0) -sg\big(
\psi^{\theta}_{0\to 1-\Delta t}(x_0) + \Delta t u_{1-\Delta t}(x^\theta_{1-\Delta t})
\big)\|^2 \Big],        
    \end{aligned}
\end{equation}
We mollify the kernel $k_t(y,x)$ as 
\begin{equation}\label{eq:mollified_kernel}
\begin{aligned}
    k^\epsilon_t(y,x) = e^{-\frac{\|ty-x\|^2}{2(1-t)^2+\epsilon}}
\end{aligned}    
\end{equation} 
to avoid singular behavior near $t=1$, and now adopt the closed-form solution in \autoref{eq:close_form1} for the instantaneous velocity at $t=1-\Delta t$ to obtain 
\begin{equation}\label{eq:forward_loss_discrete}
    \begin{aligned}
\mathcal{L}(\theta) &= \mathbb{E}_{x_0\sim p_0}\Big[
\|\psi^{\theta}_{0\to 1}(x_0) -sg\big(
\psi^{\theta}_{0\to 1-\Delta t}(x_0)+ \Delta t u_{1-\Delta t}(x^\theta_{1-\Delta t})
\big)\|^2 \Big]\\
&=\mathbb{E}_{x_0\sim p_0}\Big[
\|\psi^{\theta}_{0\to 1}(x_0) -sg\big(
\psi^{\theta}_{0\to 1-\Delta t}(x_0)+ \Delta t \frac{\mathbb{E}_{x_1\sim p_1}[\frac{x_1-x^\theta_{1-\Delta t}}{1-({1-\Delta t})}k^\epsilon_{1-\Delta t}(x_1,x^\theta_{1-\Delta t})]}{\mathbb{E}_{x_1\sim p_1}[k^\epsilon_{1-\Delta t}(x_1,x^\theta_{1-\Delta t})]}
\big)\|^2 \Big]\\
&=\mathbb{E}_{x_0\sim p_0}\Big[
\|\psi^{\theta}_{0\to 1}(x_0) -sg\big(
\psi^{\theta}_{0\to 1-\Delta t}(x_0) + \frac{\mathbb{E}_{x_1\sim p_1}[(x_1-\psi^{\theta}_{0\to 1-\Delta t}(x_0))k^\epsilon_{1-\Delta t}(x_1,\psi^{\theta}_{0\to 1-\Delta t}(x_0))]}{\mathbb{E}_{x_1\sim p_1}[k^\epsilon_{1-\Delta t}(x_1,\psi^{\theta}_{0\to 1-\Delta t}(x_0))]}
\big)\|^2 \Big]
    \end{aligned}
\end{equation}
Then let $\Delta t\to 0$, we have
\begin{equation}\label{eq:forward_euler_loss}
    \begin{aligned}
\mathcal{L}(\theta) =\mathbb{E}_{x_0\sim p_0}\Big[
\|\psi^{\theta}_{0\to 1}(x_0) -sg\big(
\psi^{\theta}_{0\to 1}(x_0)  + \frac{\mathbb{E}_{x_1\sim p_1}[(x_1-\psi^{\theta}_{0\to 1}(x_0))k^\epsilon_1(x_1,\psi^{\theta}_{0\to 1}(x_0))]}{\mathbb{E}_{x_1\sim p_1}[k^\epsilon_1(x_1,\psi^{\theta}_{0\to 1}(x_0))]}
\big)\|^2 \Big]
    \end{aligned}
\end{equation}
where $k^\epsilon_1(y,x)=e^{-\frac{\|y-x\|^2}{\epsilon}}$.  This loss involves only a single learnable map $\psi_{0\to 1}^\theta$, and therefore can be used to directly train $\psi_{0\to 1}^\theta$.  During training, the expectations in \autoref{eq:forward_euler_loss} are estimated by Monte Carlo using samples within each mini-batch.  Once learned, it naturally yields a one-step generator by sampling $x_0\sim p_0$ and computing $\psi_{0\to 1}^\theta(x_0)$.

In this loss, the residual term $ \mathcal{R}(\theta)\triangleq
\frac{\mathbb{E}_{x_1\sim p_1}\!\big[(x_1-\psi^{\theta}_{0\to 1}(x_0))\,k^\epsilon_1(x_1,\psi^{\theta}_{0\to 1}(x_0))\big]}
{\mathbb{E}_{x_1\sim p_1}\!\big[k^\epsilon_1(x_1,\psi^{\theta}_{0\to 1}(x_0))\big]} $is induced by the closed-form solution associated with an infinitesimally short terminal map $\psi^\theta_{1-\Delta t\to 1}$ as $\Delta t\to 0$, and it serves as an effective supervisory signal for learning the long map $\psi^\theta_{0\to 1-\Delta t}$.   According to our derivation, when $\psi^\theta_{0\to 1}(x_0)$ matches the desired transport, the model output and the $\mathrm{sg}(\cdot)$ target coincide, and $\mathcal{R}(\theta)$ attains its optimum. Moreover, $\mathcal{R}(\theta)$ admits a concrete interpretation: it is a regression-optimal, kernel-weighted conditional expectation that quantifies the discrepancy between the generated point $\psi^\theta_{0\to 1}(x_0)$ and the target distribution $p_1$, and the loss minimizes this discrepancy in the least-squares sense.

%This formulation reveals that, in effect, we have implicitly defined a drift field for the loss \autoref{eq:drifting_loss}
%\begin{equation}
%    \begin{aligned}
%        V_{p,q}(x) &= V_p^+(x)\\ 
%        V_p^+(x) &=\frac{\mathbb{E}_p[k^\epsilon_1(x,y^+)(y^+-x)]}{\mathbb{E}_p[k^\epsilon_1(x,y^+)]}
%     \end{aligned}
%\end{equation}

\subsection{Second-Order Approximation}\label{sec:second_order}
\begin{wrapfigure}{r}{0.5\columnwidth}
  \centering
  \vspace{-1.8em}
  \includegraphics[width=1.0\linewidth, trim=4pt 4pt 4pt 4pt, clip]{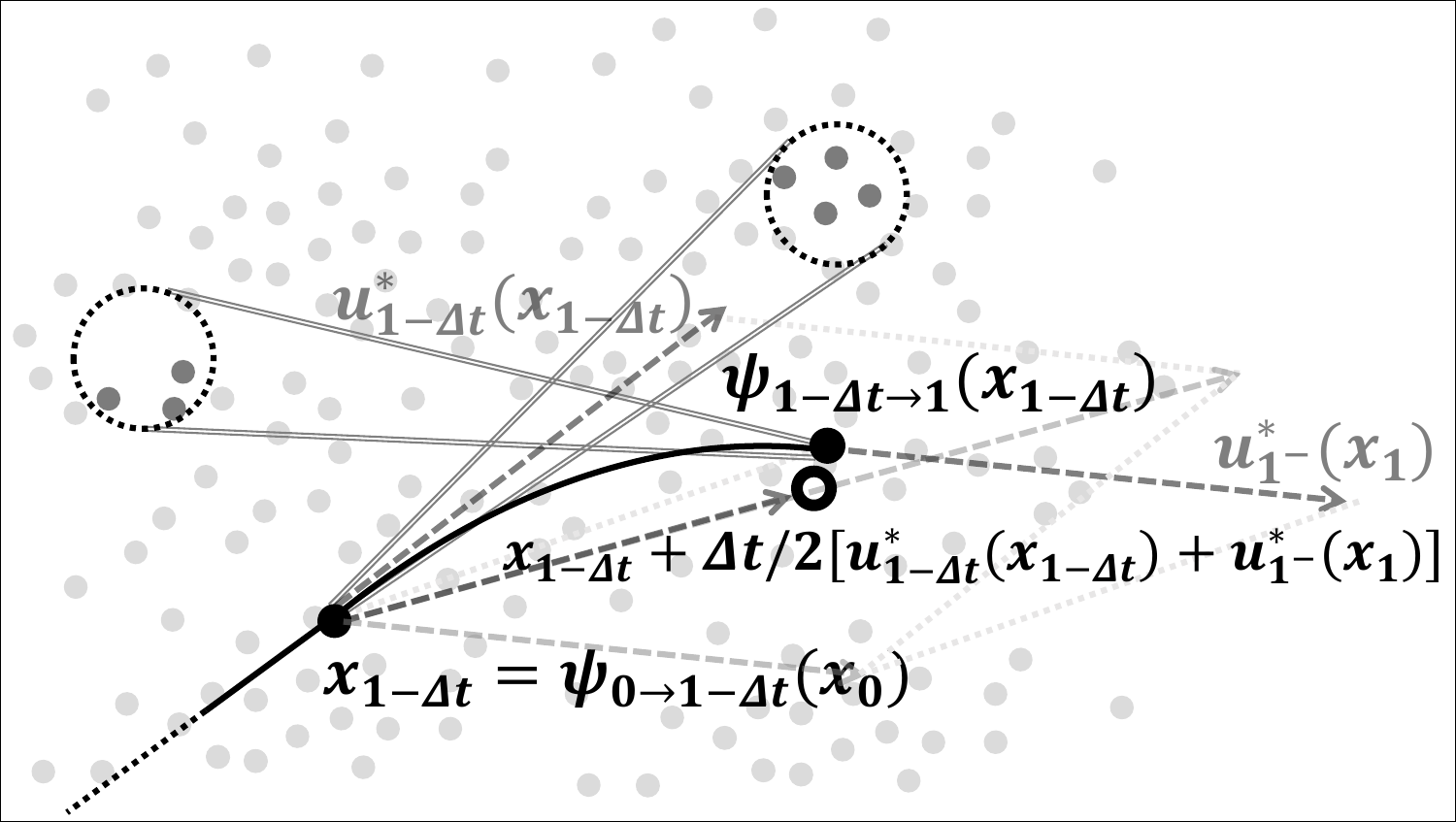}
  %\caption{\textbf{First-Order Approximation}}
  %\label{fig:first}
  \vspace{-6.0em}
\end{wrapfigure}

%\cite{deng2026generative} argue that using only $V^{+}$ leads to noticeably worse performance in their experiments. The underlying reason is that the segment $[1-\Delta t,1]$ in \autoref{eq:forward_euler_approx} is effectively approximated by a first-order forward Euler step. 
In the discussion above, we approximated the short flow map using a first-order forward Euler method and then applied the closed-form solution. To obtain a more accurate estimation, here we instead use a second-order trapezoidal rule to approximate the evolution over $[1-\Delta t,1]$.
\begin{equation}\label{eq:trapezoidal_rule}
    \begin{aligned}
        \psi^{\theta}_{0\to 1}(x_0) \approx
\psi^{\theta}_{0\to 1-\Delta t}(x_0) &+ \frac{\Delta t}{2} (u_{1-\Delta t}(\psi^{\theta}_{0\to 1-\Delta t}(x_0))\\
&+u_{1^-}(\psi^{\theta}_{0\to 1}(x_0))).
    \end{aligned}
\end{equation}
As in \autoref{eq:forward_loss_discrete}, we use the Flow Matching closed-form solution for the velocity on the interval $[1-\Delta t,1]$. The main technical issue is that we require a closed-form expression for the endpoint velocity $u_{1^-} = \lim_{t\to 1^-} u_t$, which is given by the following theorem.
\begin{theorem}[Closed-Form Solution of the Endpoint Velocity]\label{thm:close_form2}For Flow Matching, under mild smoothness assumptions on the underlying distributions, the optimal velocity satisfies $\lim_{t\to 1^-} u_t^*(x)=x$. Consequently, the left-limit endpoint velocity $u_{1^-}^*(x)\triangleq \lim_{t\to 1^-} u_t^*(x)$ admits the closed form for any $t\in [0,1)$ 
\begin{equation}   \label{eq:close_form2}
\begin{aligned}         
u^*_{1^-}(x)          = \mathbb{E}_{x_t \sim p_{t|1}(x_t|x)}[ \frac{x-x_t}{1-t}] = \frac{\mathbb{E}_{x_t\sim p_t}[\frac{x-x_t}{1-t}k_t(x,x_t)]}{\mathbb{E}_{x_t\sim p_t}[k_t(x,x_t)]}     \end{aligned} 
\end{equation}
see \autoref{sec:proof_close_form2} for proof.
\end{theorem}
In the above formulation, $p_t$ does not admit a closed-form expression. We therefore approximate it using samples pushed forward by the learned long flow map, i.e., $x_t \approx \psi_{0\to t}^{\theta}(x_0)$ with $x_0\sim p_0$.  Moreover, although the theorem states $\lim_{t\to 1^-} u_t^*(x)=x$, directly setting $u_{1^-}(x)=x$ is insufficient: it does not cancel the $\Delta t$ terms (which is crucial in the derivation of \autoref{eq:forward_loss_discrete}) and it cannot reflect the discrepancy between $p_t$ and $p_1$. We therefore adopt the Flow Matching closed-form solution of the instantaneous velocity $u_{1-\Delta t}$ and $u_{1^-}$ at $t=1-\Delta t$ and $t=1$
%, and let $x_{1-\Delta t}^\theta = \psi_{0\to 1-\Delta t}^\theta(x_0)$ and $x_{1}^\theta = \psi_{0\to 1}^\theta(x_0)$
\begin{equation}
    \begin{aligned}
\mathcal{L}(\theta) &= \mathbb{E}_{x_0\sim p_0}\Big[
\|\psi^{\theta}_{0\to 1}(x_0) -sg\big(
\psi^{\theta}_{0\to 1-\Delta t}(x_0) + \frac{1}{2}\Delta t u_{1-\Delta t}(x^\theta_{1-\Delta t})+\frac{1}{2}\Delta tu_{1^-}(x_1^\theta)
\big)\|^2 \Big]\\
&=\mathbb{E}_{x_0\sim p_0}\Big[
\|\psi^{\theta}_{0\to 1}(x_0) -sg\big(
\psi^{\theta}_{0\to 1-\Delta t}(x_0) + \frac{1}{2}\Delta t \frac{\mathbb{E}_{x_1\sim p_1}[\frac{x_1-\psi^{\theta}_{0\to 1-\Delta t}(x_0)}{1-({1-\Delta t})}k^\epsilon_{1-\Delta t}(x_1,\psi^{\theta}_{0\to 1-\Delta t}(x_0))]}{\mathbb{E}_{x_1\sim p_1}[k^\epsilon_{1-\Delta t}(x_1,\psi^{\theta}_{0\to 1-\Delta t}(x_0))]}
\\
&+\frac{1}{2}\Delta t \frac{\mathbb{E}_{x^\theta_{1-\Delta t}\sim p^\theta_{1-\Delta t}}[\frac{\psi^{\theta}_{0\to 1}(x_0)-x_{1-\Delta t}^\theta}{1-(1-\Delta t)}k^\epsilon_{1-\Delta t}(\psi^{\theta}_{0\to 1}(x_0),x_{1-\Delta t}^\theta)]}{\mathbb{E}_{x^\theta_{1-\Delta t}\sim p^\theta_{1-\Delta t}}[k^\epsilon_{1-\Delta t}(\psi^{\theta}_{0\to 1}(x_0),x_{1-\Delta t}^\theta)]}\big)\|^2 \Big]\\
&=\mathbb{E}_{x_0\sim p_0}\Big[
\|\psi^{\theta}_{0\to 1}(x_0) -sg\big(
\psi^{\theta}_{0\to 1-\Delta t}(x_0)+ \frac{1}{2}\frac{\mathbb{E}_{x_1\sim p_1}[(x_1-\psi^{\theta}_{0\to 1-\Delta t}(x_0))k^\epsilon_{1-\Delta t}(x_1,\psi^{\theta}_{0\to 1-\Delta t}(x_0))]}{\mathbb{E}_{x_1\sim p_1}[k^\epsilon_{1-\Delta t}(x_1,\psi^{\theta}_{0\to 1-\Delta t}(x_0))]}\\
&+\frac{1}{2}\frac{\mathbb{E}_{x^\theta_{1-\Delta t}\sim p^\theta_{1-\Delta t}}[(\psi_{0 \to 1}^\theta(x_0)-x_{1-\Delta t}^\theta)k^\epsilon_{1-\Delta t}(\psi_{0 \to 1}^\theta(x_0),x_{1-\Delta t}^\theta)]}{\mathbb{E}_{x^\theta_{1-\Delta t}\sim p^\theta_{1-\Delta t}}[k^\epsilon_{1-\Delta t}(\psi_{0 \to 1}^\theta(x_0),x_{1-\Delta t}^\theta)]}\big)\|^2 \Big]
    \end{aligned}
\end{equation}
where we denote the distribution of $x^\theta_{1-\Delta t}= \psi^\theta_{0\to 1-\Delta t}(x'_0),x'_0\sim p_0$ as $x^\theta_{1-\Delta t}\sim p_{1-\Delta t}^\theta$. Then let $\Delta t\to 0$, we have
\begin{equation}\label{eq:second_order_loss}
    \begin{aligned}
        \mathcal{L}(\theta) &= \mathbb{E}_{x_0\sim p_0}\Big[
\|\psi^{\theta}_{0\to 1}(x_0) -sg\big(
\psi^{\theta}_{0\to 1}(x_0) + \frac{1}{2}\frac{\mathbb{E}_{x_1\sim p_1}[(x_1-\psi^{\theta}_{0\to 1}(x_0))k^\epsilon_{1}(x_1,\psi^{\theta}_{0\to 1}(x_0))]}{\mathbb{E}_{x_1\sim p_1}[k^\epsilon_{1}(x_1,\psi^{\theta}_{0\to 1}(x_0))]}\\
&+\frac{1}{2}\frac{\mathbb{E}_{x^\theta_{1} = \psi_{0\to 1}^\theta(x'_0),x'_0\sim p_0}[(\psi_{0 \to 1}^\theta(x_0)-x_{1}^\theta)k^\epsilon_{1}(\psi_{0 \to 1}^\theta(x_0),x_{1}^\theta)]}{\mathbb{E}_{x^\theta_{1} = \psi_{0\to 1}^\theta(x'_0),x'_0\sim p_0}[k^\epsilon_{1}(\psi_{0 \to 1}^\theta(x_0),x_{1}^\theta)]}\big)\|^2 \Big]
    \end{aligned}
\end{equation}
In this loss, the trapezoidal-rule approximation of the short flow map provides a more accurate supervisory signal for learning the long flow map. Compared to the first-order Euler-based objective in \autoref{eq:forward_euler_loss}, the key difference lies in the residual term: it is no longer a point-to-distribution regression target between point $\psi_{0\to 1}^\theta(x_0)$ and distribution $p_1$, but rather a discrepancy between two distributions $p_1^\theta$ and $p_1$ estimated from samples $\psi_{0\to 1}^\theta(x_0)$. As a result, when the flow map is optimal, the residual does not merely attain a regression-optimal value; it becomes exactly zero.

%This formulation reveals that, in effect, we have implicitly defined a drift field
%\begin{equation}
%    \begin{aligned}
%        V_{p,q}(x) &= V_p^+(x) - V_q^+(x)\\ 
%        V_p^+(x) &=\frac{\mathbb{E}_p[k_1(x,y^+)(y^+-x)]}{\mathbb{E}_p[k_1(x,y^+)]} \\
%V_q^-(x) &=\frac{\mathbb{E}_q[k_1(x,y^-)(y^-x)]}%{\mathbb{E}_q[k_1(x,y^-)]}
%     \end{aligned}
%\end{equation}
%Here, $y^- \sim q = p_1^\theta$ and $y^+ \sim p = p_1$. Up to this point, we have fully derived the Drifting Model formulation. In summary, we obtain the following conclusions:

\subsection{Connection with Drifting Model}
%\bo{Introduce Drifting Model first, and then show the connection.}
%In the derivation above, we leverage the long-short flow-map decomposition to combine flow-map learning with the flow-matching closed-form solution: the closed form on the short terminal segment provides a dataset-tied supervisory signal, which in turn enables learning the long segment and yields a one-step map $\psi^\theta_{0\to 1}$. The resulting objectives in \autoref{eq:forward_euler_loss} and \autoref{eq:second_order_loss} are exactly the loss used in the Drifting Model \cite{deng2026generative}. This provides a principled derivation of the Drifting Model and offers explanations for several of its key design choices. We first summarize the Drifting Model.
In the derivation above, we use the long-short flow-map decomposition to combine flow-map learning with the Flow-Matching closed-form solution: the closed-form solution on the short flow map provides a dataset-tied supervisory signal, enabling learning the long flow map and yielding a one-step map $\psi^\theta_{0\to 1}$. From this perspective, the objectives in \autoref{eq:forward_euler_loss} and \autoref{eq:second_order_loss} recover the losses used in the Drifting Model \cite{deng2026generative}. This connection provides a principled derivation of the Drifting Model and enables intuitive explanations of its key design choices. We first summarize the Drifting Model and then elaborate on its connection to our derivation by showing how its training objectives and design choices naturally arise from the long-short flow-map viewpoint.

\paragraph{Drifting Model}
The Drifting Model \cite{deng2026generative} directly learns a one-step mapping $f^\theta:\mathcal{X}\to\mathcal{X}$ that pushes forward noise to data, i.e., $f^\theta(\epsilon)\sim q_\theta$ with $\epsilon\sim p_0$, and aims to match $q_\theta$ to $p_1=p_{\text{data}}$. The main idea is to make the pushforward distribution progressively closer to the target distribution during training. Specifically, the method introduces a drift field $V_{p_1,q_\theta}:\mathcal{X}\to\mathcal{X}$ that corrects samples from $q_\theta$ toward $p_1$ and ideally satisfies $x+V_{p,q_\theta}(x)\sim p_1$ when $x\sim q_\theta$. Training then encourages $f^\theta(\epsilon)$ to agree with its drift-corrected counterpart:
\begin{equation}\label{eq:drifting_loss}
\mathcal{L}(\theta)
= \mathbb{E}_{\epsilon\sim p_0}\Big[\big\|f^\theta(\epsilon)-
\mathrm{sg}\!\big(f^\theta(\epsilon)+V_{p_1,q_\theta}(f^\theta(\epsilon))\big)\big\|_2^2\Big],
\end{equation}
where $q_\theta$ denotes the distribution induced by $f^\theta(\epsilon)$.

For $V_{p,q}$ that incorporates supervision from the target data distribution, the Drifting Model adopts a kernel-based construction with an attraction-repulsion structure:
\begin{equation}
\begin{aligned}
V_{p,q}(x)&= V_p^+(x)-V_q^-(x),\\
V_p^+(x)&=\frac{\mathbb{E}_{y\sim p}[k(x,y)(y-x)]}{\mathbb{E}_{y\sim p}[k(x,y)]},\\
V_q^-(x)&=\frac{\mathbb{E}_{y\sim q}[k(x,y)(y-x)]}{\mathbb{E}_{y\sim q}[k(x,y)]},    
\end{aligned}
\end{equation}
where $k(x,y)$ is a similarity kernel (implemented via a Laplacian kernel $e^{-\|y-x\|/\epsilon}$ in practice). Ideally, this definition satisfies a consistency condition: $V_{p,q}(x)=0$ whenever $p=q$, so that no correction is applied when the two distributions coincide.

%In practice, the expectations above are approximated by Monte Carlo using mini-batches. At each iteration, $n$ noise vectors $\epsilon^1,\ldots,\epsilon^n\sim p_0$ and $m$ data samples $x^1,\ldots,x^m\sim p_{\text{data}}$ are sampled. For any input $x$, the drift field is estimated as
%\begin{equation}
%\resizebox{\linewidth}{!}{$
%\label{eq:drift_mc}
%V_{p,q}(x)
%=
%\frac{\sum_{j=1}^m k(x,x^j)(x^j-x)}{\sum_{j=1}^m k(x,x^j)}
%-
%\frac{\sum_{i=1}^n k(x,f^\theta(\epsilon^i))(f^\theta(\epsilon^i)-x)}{\sum_{i=1}^n k(x,f^\theta(\epsilon^i))}.
%$}
%\end{equation}
%Substituting \autoref{eq:drift_mc} into the objective yields the empirical training loss
%\begin{equation}
%\mathcal{L}(\theta)
%=
%\sum_{k=1}^n
%\Big\|
%f^\theta(\epsilon_k)
%-
%\mathrm{sg}\Big(
%f^\theta(\epsilon_k) + %V_{p,q_\theta}\big(f^\theta(\epsilon^k)\big)
%\Big)
%\Big\|_2^2,
%\end{equation}

\paragraph{Connection of Loss Functions}
We observe that the loss in \autoref{eq:forward_euler_loss} obtained by approximating the short terminal evolution via a first-order forward Euler step coincides with the Drifting Model objective that uses the attraction-only drift field $V_{p,q}(x)=V_p^+(x)$ for the drifting loss in \autoref{eq:drifting_loss}, where
\begin{equation}
    \begin{aligned}
        f^\theta \coloneqq \psi_{0\to 1}^\theta,\qquad V_{p_1,p_1^\theta} = V_{p_1}^+,\qquad V_{p_1}^+(x) = \frac{\mathbb{E}_{x_1\sim p_1}[(x_1-x)k_1(x_1,x)]}{\mathbb{E}_{x_1\sim p_1}[k_1(x_1,x)]} 
    \end{aligned}
\end{equation}
In contrast, the loss in \autoref{eq:second_order_loss} derived from a second-order trapezoidal approximation over the same short interval matches the Drifting Model objective with the attraction-repulsion drift field $V_{p,q}(x)=V_p^+(x)-V_q^-(x)$, where
\begin{equation}
    \begin{aligned}
        f^\theta \coloneqq \psi_{0\to 1}^\theta,\quad V_{p_1,p_1^\theta} = V_{p_1}^+-V_{p_1^\theta}^-,\quad V_{p_1}^+(x) = \frac{\mathbb{E}_{x_1\sim p_1}[(x_1-x)k_1(x_1,x)]}{\mathbb{E}_{x_1\sim p_1}[k_1(x_1,x)]}, \quad V_{p_1^\theta}^-(x) = \frac{\mathbb{E}_{x_1\sim p^\theta_1}[(x_1-x)k_1(x,x_1)]}{\mathbb{E}_{x_1\sim p^\theta_1}[k_1(x,x_1)]} 
    \end{aligned}
\end{equation}

\paragraph{Discussion on Design Choices} The connection between the loss functions yields the following insights regarding several design choices of the Drifting Model, which we will discuss briefly below.
\begin{enumerate}
    \item Interpretation. The Drifting Model can be viewed as learning a flow map via trajectory consistency, and then selecting the segment $s \to r$ in \autoref{prop:trajector_consistency} to be a very short interval near the terminal time, i.e., $1-\Delta t \to 1$. Over this short interval, it directly applies the closed-form solution of Flow Matching to obtain supervision from the dataset.
   \item  Why $V^+ - V^-$ emerges and why it helps. If the estimation over $1-\Delta t \to 1$ uses a first-order forward Euler method, one recovers the original Drifting Model form that involves only $V^+$. In contrast, if a second-order trapezoidal rule is used, the derivation naturally yields the balanced attraction–repulsion form $V^+ - V^-$. Since the learned flow map is generally not optimal, the terminal interval $1-\Delta t \to 1$ remains affected by approximation error even as $\Delta t \to 0$. Therefore, the order of the terminal-interval estimator provides a principled explanation for why the symmetric correction $V^+ - V^-$ tends to outperform the one-sided variant $V^+$, as well as other imbalanced choices such as $3V^+ - V^-$.
    \item Kernel Selection. As revealed by the proofs of \autoref{sec:proof_close_form1} and \autoref{sec:proof_close_form2}, the kernel $k_t(x,y)$ adopted in \autoref{eq:forward_euler_loss} and \autoref{eq:second_order_loss} is tightly coupled with the form of the initial noise distribution $p_0$. In particular, when we choose Gaussian noise, the resulting kernel naturally becomes the squared-exponential Gaussian kernel $e^{-|y-x|^2/\epsilon}$. Importantly, our framework does not restrict the choice of $p_0$, and therefore does not constrain the kernel family in \autoref{eq:forward_euler_loss} and \autoref{eq:second_order_loss} either. For example, by choosing $p_0$ to follow a Laplace distribution, we recover the Laplacian kernel $e^{-|y-x|/\epsilon}$ used in the original Drifting Model \cite{deng2026generative}. More broadly, relating kernel design to the choice of initial noise indicates a connection between these two components. This connection hints that recent advances in initial-noise optimization, such as blue-noise sampling strategies \cite{huang2024blue}, could be incorporated to refine the kernel selections in \autoref{eq:forward_euler_loss} and \autoref{eq:second_order_loss}. A systematic investigation of this direction remains open.

\end{enumerate}
In \autoref{sec:CFG}, we discuss classifier-free guidance (CFG), which \cite{deng2026generative} identifies as a key distinction between Drifting Model and distribution moment matching \cite{li2015generative, dziugaite2015training} (see the related-work discussion in \autoref{sec:related_work}).

\section{Application: Likelihood Learning}
Here we leverage the long-short flow-map perspective to derive a new objective for learning data likelihoods.

\subsection{Background}
For one-step generative methods based on flow maps, one can in principle compute the likelihood using the learned map $\psi_{t\to r}(x)$ via the change-of-variables formula $\log p_r(x_r) = \log p_t(x_t)-\log|\det \frac{\partial \psi_{t\to r}(x)}{\partial x}|_{x=x_t}$, where $x_r=\psi_{t\to r}(x_t)$.  However, this requires evaluating the Jacobian determinant of the flow map, which does not scale to high-dimensional data.  To address this issue, \cite{ai2025joint} proposes to jointly learn, together with the flow map $\psi^\theta_{t\to r}(x)$, an auxiliary model $D^{\theta}_{t\to r}(x)$ that predicts the average likelihood change from time $t$ to $r$, namely $D_{t\to r}(x) = \frac{\log p_r(\psi_{t\to r}(x)) - \log p_t(x)}{r-t}$ . Concretely, the evolution of the log-likelihood quantity $\log p_r(\psi_{t\to r}(x))$ satisfies
\begin{equation}\label{eq:likelihood_evolv}
    \begin{aligned}
        \frac{\partial \log p_r(\psi_{t\to r}(x))}{\partial r} = -\nabla\cdot u_r(y)|_{y = \psi_{t\to r}(x)}
    \end{aligned}
\end{equation}
and $D_{t\to r}(x)$ satisfies a trajectory consistency property $(r-t)D_{t\to r}(x_t) = (r-s) D_{s\to r}(x_s) + (s-t) D_{t\to s}(x_t)$, where $x_s = \psi_{t\to s}(x_t)$ for arbitrary $t< s< r$.  Using this trajectory consistency and the evolution \autoref{eq:likelihood_evolv} as a boundary condition $D_{t\to t}(x) = -\nabla\cdot u_t(x)$, \cite{ai2025joint} learns $D_{t\to r}(x)$ jointly with the flow map $\psi_{t\to r}$.  With the learned $D^{\theta}_{t\to r}(x)$, for any generated sample $x_1 = \psi_{0\to 1}(x_0)$, one can efficiently obtain its likelihood $\log p_1(x_1) = \log p_0(x_0) + D^{\theta}_{0\to 1}(x_0)$.

\subsection{Drifting Model for Likelihood Learning}

We aim to derive an analogous likelihood-learning objective for the Drifting Model under our flow-map perspective.  Different from \cite{ai2025joint}, our goal is to learn the Eulerian log-likelihood change $G_{t\to r}(x)=\log p_r(x)-\log p_t(x)$.
When combined with the flow map $\psi_{t\to r}$, the composed quantity $G_{t\to r}(\psi_{t\to r}(x))$ can also play a Lagrangian role.  In the Eulerian view, $\log p_t(x)$ satisfies the following evolution equation
\begin{equation}
    \begin{aligned}
        \frac{\partial}{\partial t}\log p_t(x) &=-\nabla\log p_t(x)\cdot u_t(x) -\nabla\cdot u_t(x)
    \end{aligned}
\end{equation}
Because $G_{t\to r}(x) = \log p_r(x)-\log p_t(x)$, it satisfies the evolution equation $\frac{\partial}{\partial r} G_{t\to r}(x) =-\nabla\log p_r(x)\cdot u_r(x) -\nabla\cdot u_r(x)$ and the consistency relation $G_{t\to r}(x) = G_{s\to r}(x) + G_{t\to s}(x)$ for arbitrary $t< s< r$.  Following the same idea as before, we apply this relation on $[0,1]$ with a split point $1-\Delta t$, namely $G_{0\to 1}(x) = G_{0\to 1-\Delta t}(x) + G_{1-\Delta t\to 1}(x)$ and approximate $G_{1-\Delta t\to 1}(x)$ by the trapezoidal rule. For the learned flow map $\psi_{t\to r}^\theta(x)$ and the learned likelihood change $G_{t\to r}^\theta(x)$, we enforce that they satisfy
\begin{equation}
    \begin{aligned}
        G^\theta_{0\to 1}(x) &\approx G^\theta_{0\to 1-\Delta t}(x) - \frac{\Delta t}{2}(\nabla \log p^\theta_{1-\Delta t}(x)u_{1-\Delta t}(x) + \nabla \log p^\theta_{1}(x)u_{1^-}(x)+\nabla\cdot u_{1^-}(x) + \nabla\cdot u_{1-\Delta t}(x))
    \end{aligned}
\end{equation} 
where $p_t^\theta(x) = p_0(x)e^{G_{0\to t}^\theta(x)} $.  Here, we need to compute the divergence terms $\nabla \cdot u_{1-\Delta t}$ and $\nabla \cdot u_{1^-}$, whose evaluation is provided by the following theorem.

\begin{theorem}[Closed-Form Solution for Velocity Divergence]\label{thm:close_form_divergence}
For the closed-form solution of the velocity field $u^*_t(x)$ in \autoref{eq:close_form1} and the endpoint velocity $u^*_{1^-}(x)$ in \autoref{eq:close_form2} with the mollified kernel $k_t^\epsilon$ in \autoref{eq:mollified_kernel}, their divergences can be also computed in closed form as
\begin{equation}
    \begin{aligned}
        \nabla \cdot u^*_t(x) &= \frac{1}{1-t}[\frac{2t}{2(1-t)^2+\epsilon}(\frac{\mathbb{E}_{x_1\sim p_1}[\|x_1\|^2k^\epsilon_{t}(x_1,x)]}{\mathbb{E}_{x_1\sim p_1}[k^\epsilon_{t}(x_1,x)]}-\|\frac{\mathbb{E}_{x_1\sim p_1}[x_1k^\epsilon_{t}(x_1,x)]}{\mathbb{E}_{x_1\sim p_1}[k^\epsilon_{t}(x_1,x)]}\|^2)-d]\\
        \nabla\cdot u^*_{1^-}(x) &= \frac{1}{1-t}[d-\frac{2t^2}{2(1-t)^2+\epsilon}(\frac{\mathbb{E}_{x_t\sim p_t}[\|x_t\|^2k^\epsilon_{t}(x,x_t)]}{\mathbb{E}_{x_t\sim p_t}[k^\epsilon_{t}(x,x_t)]}-\|\frac{\mathbb{E}_{x_t\sim p_t}[x_tk^\epsilon_{t}(x,x_t)]}{\mathbb{E}_{x_t\sim p_t}[k^\epsilon_{t}(x,x_t)]}\|^2)]
    \end{aligned}
\end{equation}
See \autoref{sec:proof_close_form_divergence} for the proof.
\end{theorem}

Therefore, $G^\theta_{0\to 1}(x)$ can be calculated as
\begin{equation}
    \begin{aligned}
        G^\theta_{0\to 1}(x) &= G^\theta_{0\to 1-\Delta t}(x) - \frac{\Delta t}{2}(\nabla \log p_{1-\Delta t}^\theta (x)\cdot u_{1-\Delta t}(x) +\nabla \log p^\theta_1(x)\cdot u_{1^-}(x)+ \nabla\cdot u_{1^-}(x) + \nabla\cdot u_{1-\Delta t}(x))\\
        &= G^\theta_{0\to 1-\Delta t}(x) -\frac{1}{2}(\nabla \log p^\theta_{1-\Delta t}(x)\frac{\mathbb{E}_{x_1\sim p_1}[(x_1-x)k_{1-\Delta t}^\epsilon(x_1,x)]}{\mathbb{E}_{x_1\sim p_1}[k_{1-\Delta t}^\epsilon(x_1,x)]}\\&+\nabla \log p^\theta_1(x)\cdot \frac{\mathbb{E}_{x_{1-\Delta t}\sim p_{1-\Delta t}}[(x-x_{1-\Delta t}) k_{1-\Delta t}^\epsilon(x,x_{1-\Delta t})]}{\mathbb{E}_{x_{1-\Delta t}\sim p_{1-\Delta t}}[k_{1-\Delta t}^\epsilon(x,x_{1-\Delta t})]})\\
        &- \frac{(1-\Delta t)}{2\Delta t^2+\epsilon}(\frac{\mathbb{E}_{x_1\sim p_1}[\|x_1\|^2k^\epsilon_{1-\Delta t}(x_1,x)]}{\mathbb{E}_{x_1\sim p_1}[k^\epsilon_{1-\Delta t}(x_1,x)]}-\|\frac{\mathbb{E}_{x_1\sim p_1}[x_1k^\epsilon_{1-\Delta t}(x_1,x)]}{\mathbb{E}_{x_1\sim p_1}[k^\epsilon_{1-\Delta t}(x_1,x)]}\|^2)\\
        &+\frac{(1-\Delta t)^2}{2\Delta t^2+\epsilon}(\frac{\mathbb{E}_{x_{1-\Delta t}\sim p_{1-\Delta t}}[\|x_{1-\Delta t}\|^2k^\epsilon_{1-\Delta t}(x,x_{1-\Delta t})]}{\mathbb{E}_{x_{1-\Delta t}\sim p_{1-\Delta t}}[k^\epsilon_{1-\Delta t}(x,x_{1-\Delta t})]}-\|\frac{\mathbb{E}_{x_{1-\Delta t}\sim p_{1-\Delta t}}[x_{1-\Delta t}k^\epsilon_{1-\Delta t}(x,x_{1-\Delta t})]}{\mathbb{E}_{x_{1-\Delta t}\sim p_{1-\Delta t}}[k^\epsilon_{1-\Delta t}(x,x_{1-\Delta t})]}\|^2)\\
    \end{aligned}
\end{equation}
In the above expression, since $p_{1-\Delta t}$ is unknown, we instead draw samples from $p^\theta_{1-\Delta t}$ induced by $G^\theta_{0\to 1-\Delta t}$; the sampling procedure is described in \autoref{sec:sampling}. Letting $\Delta t\to 0$, we obtatin the loss
\begin{equation}\label{eq:likelihood_loss}
    \begin{aligned}
        &\mathcal{L}(\theta) = \mathbb{E}_{x\sim p_{sample}}[\|
        G_{0\to 1}^\theta(x) - sg(G_{0\to 1}^\theta(x) - 
        \frac{1}{2}\nabla G_{0\to 1}^\theta(x)(\frac{\mathbb{E}_{x_1\sim p_1}[x_1k_{1}^\epsilon(x_1,x)]}{\mathbb{E}_{x_1\sim p_1}[k_{1}^\epsilon(x_1,x)]}-\frac{\mathbb{E}_{x_1\sim p^\theta_1}[x_1k_{1}^\epsilon(x,x_1)]}{\mathbb{E}_{x_1\sim p^\theta_1}[k_{1}^\epsilon(x,x_1)]})\\
        &-\frac{1}{\epsilon}(\frac{\mathbb{E}_{x_1\sim p_1}[\|x_1\|^2k^\epsilon_{1}(x_1,x)]}{\mathbb{E}_{x_1\sim p_1}[k^\epsilon_{1}(x_1,x)]}-\|\frac{\mathbb{E}_{x_1\sim p_1}[x_1k^\epsilon_{1}(x_1,x)]}{\mathbb{E}_{x_1\sim p_1}[k^\epsilon_{1}(x_1,x)]}\|^2-\frac{\mathbb{E}_{x_{1}\sim p_{1}^\theta}[\|x_{1}\|^2k^\epsilon_{1}(x,x_{1})]}{\mathbb{E}_{x_{1}\sim p_{1}^\theta}[k^\epsilon_{1}(x,x_{1})]}+\|\frac{\mathbb{E}_{x_{1}\sim p^\theta_{1}}[x_{1}k^\epsilon_{1}(x,x_{1})]}{\mathbb{E}_{x_{1}\sim p_{1}^\theta}[k^\epsilon_{1}(x,x_{1})]}\|^2))\|^2]
    \end{aligned}
\end{equation}
where $p_{\text{sample}}$ specifies the region on which we want to learn the field. For example, if we aim to learn $G_{0\to 1}^\theta(x)$ over a bounded set $S$, we can choose $p_{\text{sample}}$ to be the uniform distribution on $S$. If we instead focus on points near the data manifold, we can set $p_{\text{sample}}$ as a mixture of Gaussians centered at $\psi_{t\to r}^\theta(x_0)$ with $x_0\sim p_0$.

In addition, to ensure that the learned likelihood satisfies the normalization constraint, we introduce a normalization loss 
\begin{equation}\label{eq:likelihood_loss2}
    \begin{aligned}
       \mathcal{L}_{\mathrm{norm}}(\theta) = \mathbb{E}_{x \sim p_{sample}}\|\mathbb{E}_{x_1\sim p_1}[k^\epsilon_{1}(x_1,x)]- \mathbb{E}_{x_1\sim p_1^\theta}[k^\epsilon_{1}(x, x_1)]\|^2 
    \end{aligned}
\end{equation}
As a result, the final training objective is given by $\mathcal{L}(\theta) + \lambda \mathcal{L}_{\mathrm{norm}}(\theta)$ with hyperparameter $\lambda$.  The expectations in \autoref{eq:likelihood_loss} and \autoref{eq:likelihood_loss2} are approximated during training by mini-batch Monte Carlo estimates, as in the sample-generation losses in \autoref{eq:forward_euler_loss} and \autoref{eq:second_order_loss}.

\subsection{Sampling Process for  Lagrangian Likelihood Learning}\label{sec:sampling}
In the above formulation, we note that the expectation
$\mathbb{E}_{x_1 \sim p_1^\theta}[\cdot]$ is required. Although $p_1^\theta(x) = p_0(x)e^{D^\theta_{0\to 1}(x)}$,  the model $D^\theta_{0\to 1}(x_0)$ is not directly
amenable to sampling. To address this issue, we adopt importance sampling,
i.e.,
\begin{equation}
    \begin{aligned}
        \mathbb{E}_{x_1 \sim p_1^\theta}[\cdots]
        =
        \mathbb{E}_{x_1 \sim p_{\mathrm{ref}}}
        \left[
            \cdots            \frac{p_0(x_1)e^{D^\theta_{0\to 1}(x_1)}}{p_{\mathrm{ref}}(x_1)}
        \right].
    \end{aligned}
\end{equation}
Here, $p_{\mathrm{ref}}$ denotes a reference distribution from which we can
efficiently sample and whose probability density function is available.
Ideally, the support of $p_{\mathrm{ref}}$ should closely match that of
$p_1^\theta$.

In practice, we construct $p_{\mathrm{ref}}$ on a per-batch basis as a multimodal Gaussian mixture whose components are centered at the mini-batch samples $\{x_1^{(i)}\}_{i=1}^{B}$ drawn from the dataset. Specifically, we define
\begin{equation}
    \begin{aligned}
p_{\mathrm{ref}}(x)=\frac{1}{B}\sum_{i=1}^{B}\mathcal{N}\left(x;x_1^{(i)},\sigma^2 I\right).
    \end{aligned}
\end{equation}
where $B$ denotes the batch size and $\sigma$ controls the spread of each
Gaussian component.

\section{Open Problems in Feature-Space Optimization}\label{sec:feature_space}
For high-dimensional natural data, such as pixel-space generation, the Drifting Model may suffer from the curse of dimensionality. This issue arises because the training process is driven by the kernel $k_1^\epsilon(y,x)$, whose evaluation depends on pairwise distances. In high dimensions, standard distance metrics lose discriminative power. Distances concentrate and become nearly constant, which is a well-known manifestation of the dimensionality curse. As a result, the kernel weights become less informative and less stable.  A natural remedy is to evaluate the kernel in a feature space where neighborhood structure is more meaningful, for example, a lower-dimensional space, or a feature space with stronger anisotropy and effectively lower intrinsic dimension. In such spaces, pairwise distances better reflect semantic similarity and yield more stable kernel weights. Motivated by this, \cite{deng2026generative} proposes to compute the Drifting Model directly in feature space, leading to the following formulation:
\begin{equation}\label{eq:drifting_model_feature_space}
    \begin{aligned}
    \mathcal{L}(\theta) = \mathbb{E}_{\epsilon\sim p_0}[\|\phi_j(f^\theta(\epsilon))- sg(\phi_j(f^\theta(\epsilon)) +[V_{\phi_j(p_1),\phi_j(q_\theta)}(\phi_j(f^\theta(\epsilon))]))\|^2]  
    \end{aligned}
\end{equation}
where $\phi_j(\cdot)$ denotes the $j$-th feature mapping, and $\phi_j(q_\theta)$ and $\phi_j(p_1)$ denote the pushforward distribution of $q_\theta$ and $p_1$ under $\phi_j$, namely $\phi_j(f^\theta(x))$ with $x \sim p_0$ and $\phi_j(x)$ with $x \sim p_1$, respectively. Here, $p_1$ denotes the data distribution. The operator $V_{\phi_j(p),\phi_j(q_\theta)}(\phi_j(f^\theta(\epsilon)))$ is computed in the feature space $\mathcal{Z} = \phi_j(\mathcal{X})$. The collection $\{\phi_j\}_{j=1}^m$ is extracted from a pretrained, frozen encoder, and \cite{deng2026generative} uses the encoder of MAE \cite{he2022masked}. In their implementation, $m$ is on the order of several thousand.

We now interpret the Drifting Model in feature space and provide an intuitive explanation for why multiple feature maps are needed. Following the same decomposition of $\psi_{0\to 1}$ as before, we apply it in the feature space induced by $\phi_j(\cdot)$ and linearize $\phi_j$ and approximate $\psi^\theta_{1-\Delta t\to 1}(x^\theta_{1-\Delta t})$ using the trapezoidal rule as before:
\begin{equation}\label{eq:feature_discrete}
    \begin{aligned}
    \phi_j(\psi^\theta_{0\to 1}(x_0)) &= \phi_j(\psi^\theta_{0\to 1-\Delta t}(x_0) + \psi^\theta_{1-\Delta t\to 1}(x^\theta_{1-\Delta t})) \\
     &\approx \phi_j(\psi^\theta_{0\to 1-\Delta t}(x_0)) + \frac{\Delta t}{2}[J_{\phi_j}(x^\theta_{1-\Delta t})u_{1-\Delta t}(x^\theta_{1-\Delta t})+ J_{\phi_j}(x^\theta_{1})u_{1^-}(x^\theta_{1})]
    \end{aligned}
\end{equation}
where $J_{\phi_j}(x)=\frac{\partial \phi_j(x)}{\partial x}$ denotes the Jacobian of $\phi_j$. Substituting the expressions of $u_{1-\Delta t}$ and $u_{1^-}$ into the above and taking $\Delta t\to 0$ yields the limiting loss
\begin{equation}\label{eq:pricinple_feature_loss}
    \begin{aligned}
        \mathcal{L}(\theta) &= \mathbb{E}_{x_0\sim p_0}\Big[
\|\phi_j(\psi^{\theta}_{0\to 1}(x_0)) -sg\big(
\phi_j(\psi^{\theta}_{0\to 1}(x_0))+ \frac{1}{2}J_{\phi_j}(\psi^{\theta}_{0\to 1}(x_0))(\frac{\mathbb{E}_{x_1\sim p_1}[(x_1-\psi^{\theta}_{0\to 1}(x_0))k^\epsilon_{1}(x_1,\psi^{\theta}_{0\to 1}(x_0))]}{\mathbb{E}_{x_1\sim p_1}[k^\epsilon_{1}(x_1,\psi^{\theta}_{0\to 1}(x_0))]}\\
&+\frac{\mathbb{E}_{x^\theta_{1} = \psi_{0\to 1}^\theta(x'_0),x'_0\sim p_0}[(\psi_{0 \to 1}^\theta(x_0)-x_{1}^\theta)k^\epsilon_{1}(\psi_{0 \to 1}^\theta(x_0),x_{1}^\theta)]}{\mathbb{E}_{x^\theta_{1} = \psi_{0\to 1}^\theta(x'_0),x'_0\sim p_0}[k^\epsilon_{1}(\psi_{0 \to 1}^\theta(x_0),x_{1}^\theta)]})\big)\|^2 \Big]
    \end{aligned}
\end{equation}
This expression characterizes the principled optimization direction for feature space. However, $\mathcal{L}(\theta)$ still relies on kernels computed in the original space $\mathcal{X}$, which limits its ability to reduce the effective dimensionality via $\phi_j$.

\begin{challenge}[Challenge of Feature-Space Optimization]
Although $\mathcal{L}(\theta)$ in \autoref{eq:pricinple_feature_loss} provides a principled feature-space training objective, its kernel weights are still defined over $\mathcal{X}$ rather than $\phi_j(\mathcal{X})$. As a result, the neighborhood structure used for weighting remains high-dimensional, and the expected mitigation of the curse of dimensionality from operating in feature space is not realized.
\end{challenge}

In the derivations below, for convenience, we adopt the random-variable notation from Flow Matching.  From the proofs in \autoref{sec:proof_close_form1} and \autoref{sec:proof_close_form2}, we observe that the kernel weight $k_t^\epsilon$ in the closed-form optimal velocity comes from applying Flow Matching on $p_1$ with the linear conditional path $X_t=tX_1+(1-t)X_0$, which implies $\frac{X_t-tX_1}{1-t}=X_0\sim p_0$, where $X_1\sim p_1$, $X_0\sim p_0$, and $X_t\sim p_t$ are the random variables induced by the conditional probability path in Flow Matching. Therefore, if we want the kernel to be computed on $\phi_j(\mathcal{X})$, we should perform Flow Matching directly in that feature space.

Concretely, let $Z_1=\phi_j(X_1)\in\mathcal{Z}=\phi_j(\mathcal{X})$ and denote its induced data distribution by $p_1^{z}(z)=\int_{\mathcal{X}} \mathbf{1}(z=\phi_j(x))p_1(x)dx$.  We consider a Gaussian reference distribution $Z_0\sim p_0^{z}$ on $\mathcal{Z}$, and apply Flow Matching between $p_0^{z}$ and $p_1^{z}$ with the conditional path $Z_t=tZ_1+(1-t)Z_0$. Then, for any $z=\phi_j(x)$, the optimal velocity in $\mathcal{Z}$ is
\begin{equation}\label{eq:feature_space_close_form}
\begin{aligned}
u_{t}^{z,*}(z)
&=
\frac{\mathbb{E}_{z_1\sim p_1^{z}}\!\Big[\frac{z_1-z}{1-t}k_t(z_1,z)\Big]}
{\mathbb{E}_{z_1\sim p_1^{z}}\!\big[k_t(z_1,z)\big]}.\\
u_{1^-}^{z,*}(z)
&=
\frac{\mathbb{E}_{z_t\sim p_t^{z}}\!\Big[\frac{z-z_t}{1-t}k_t(z,z_t)\Big]}
{\mathbb{E}_{z_t\sim p_t^{z}}\!\big[k_t(z,z_t)\big]}.
\end{aligned}
\end{equation}
Here, $u_t^{z,*}$ and $u_{1^-}^{z,*}$ involves kernel evaluation in feature space, which matches our goal. Next, we relate $u_t^{z,*}$ to \autoref{eq:feature_discrete} and show how it can be incorporated into the estimator in \autoref{eq:feature_discrete} to obtain a feature-space Drifting Model.
\begin{theorem}[Feature-Data Space Optimal Velocity Relation]\label{thm:feature_data_optimal}
Let $X_1\sim p_1$ and define $Z_1=\phi_j(X_1)$. Denote by $u_t^{x,*}(x)$ the optimal Flow Matching velocity in data space associated with the conditional path $X_t=tX_1+(1-t)X_0$, and by $u_t^{z,*}(z)$ the optimal velocity in feature space associated with $Z_t=tZ_1+(1-t)Z_0$, where $X_0$ and $Z_0$ are Gaussian reference variables on $\mathcal{X}$ and $\mathcal{Z}$, respectively.  We define $p_{x|z}(x|z)$ as the conditional distribution induced by the inverse of the  deterministic map $z=\phi_j(x)$.  Then, for $z=\phi_j(x)$, the two velocities satisfy the following approximate relations:
\begin{equation}
\begin{aligned}
u_t^{z,*}(z) &\approx J_{\phi_j}(x)\frac{\mathbb{E}[X_1| Z_t=\phi_j(x)]-x}{1-t},\\
u_{1^-}^{z,*}(z) &\approx J_{\phi_j}(x)\frac{x-\mathbb{E}[X'_t|Z_1=\phi_j(x)]}{1-t},
\end{aligned}
\end{equation}
where $J_{\phi_j}(x)=\frac{\partial \phi_j(x)}{\partial x}$ is the Jacobian of $\phi_j$, and $X'_t$ is an auxiliary $\mathcal X$-valued random variable satisfying $\phi_j(X'_t)=Z_t$ almost surely.  The approximation $\approx$ above follows from the first-order Taylor expansion
$\phi_j(y)-\phi_j(x)=J_{\phi_j}(x)(y-x)+O(\|y-x\|^2)$. Only in the special case where $\phi_j(X_t)=Z_t$ holds (for instance, when $\phi_j$ is linear and the feature-space path $Z_t$ is induced by pushing forward the data-space path $X_t$), the optimal feature-space velocity can be written as 
\begin{equation}\label{eq:feature_space_close_form}
    \begin{aligned}
 u_t^{z,*}(z) &\approx \mathbb{E}_{x\sim p_{x|z}}[ J_{\phi_j}(x)u_t^{x,*}(x)]     \\
 u_{1^-}^{z,*}(z) &\approx \mathbb{E}_{x\sim p_{x|z}} [J_{\phi_j}(x)u_{1^-}^{x,*}(x)]
    \end{aligned}
\end{equation} 
see \autoref{sec:proof_feature_data_optimal} for proof.
 \end{theorem}

Note that, in general, we do {not} have
$u_t^{z,*}(z) = \mathbb{E}_{x\sim p_{x|z}}[ J_{\phi_j}(x)u_t^{x,*}(x)]$ and $u_{1^-}^{z,*}(z) = \mathbb{E}_{x\sim p_{x|z}} [J_{\phi_j}(x)u_{1^-}^{x,*}(x)]$.  The discrepancy comes from the nonlinearity of $\phi_j$ (e.g., the encoder of MAE model).  Nevertheless, deriving the feature-space Drifting objective in \autoref{eq:drifting_model_feature_space} requires precisely the approximation
$u_t^{z,*}(z)\approx \mathbb{E}_{x\sim p_{x|z}}[ J_{\phi_j}(x)u_t^{x,*}(x)]$ and $u_{1^-}^{z,*}(z) \approx \mathbb{E}_{x\sim p_{x|z}} [J_{\phi_j}(x)u_{1^-}^{x,*}(x)]$.
Under this assumption, taking the conditional expectation $\mathbb{E}_{x\sim p_{x| z}}$ on both sides of \autoref{eq:feature_discrete},  substituting \autoref{eq:feature_space_close_form} and substituting the closed-form expressions of $u_t^{z,*}(z)$ and $u_{1^-}^{z,*}(z)$, and letting $\Delta t\to 0$ yields
\begin{equation}
    \begin{aligned}
        \mathcal{L}(\theta) &= \mathbb{E}_{x_0\sim p_0}\mathbb{E}_{x^\theta \sim p_{x|z}(x^\theta|\phi_j(\psi^{\theta}_{0\to 1}(x_0)))}\Big[
\|\phi_j(x^\theta) -sg\big(
\phi_j(x^\theta)+ \frac{1}{2}(\frac{\mathbb{E}_{x_1\sim p_1}[(\phi_j(x_1)-\phi_j(x^\theta))k^\epsilon_{1}(\phi_j(x_1),\phi_j(x^\theta))]}{\mathbb{E}_{x_1\sim p_1}[k^\epsilon_{1}(\phi_j(x_1),\phi_j(x^\theta))]}\\
&+\frac{\mathbb{E}_{x^\theta_{1} = \psi_{0\to 1}^\theta(x'_0),x'_0\sim p_0}[(\phi_j(x^\theta)-\phi_j(x_{1}^\theta))k^\epsilon_{1}(\phi_j(x^\theta),\phi_j(x_{1}^\theta))]}{\mathbb{E}_{x^\theta_{1} = \psi_{0\to 1}^\theta(x'_0),x'_0\sim p_0}[k^\epsilon_{1}(\phi_j(x^\theta),\phi_j(x_{1}^\theta))]})\big)\|^2 \Big]
    \end{aligned}
\end{equation}
which recovers the feature-space Drifting Model loss in \autoref{eq:drifting_model_feature_space}.

The above derivation suggests the following two takeaways:
\begin{enumerate}
    \item \textbf{Dimensionality Considerations in Feature-Space Drifting.}
    Drifting in high-dimensional spaces $\mathcal{X}$ faces the challenge of the curse of dimensionality, which motivates the model to operate in a feature space $\mathcal{Z}$. However, the direct objective defined in the feature space (see \autoref{eq:pricinple_feature_loss}) still evaluates kernels in the original space $\mathcal{X}$ and, as a result, may not fully alleviate the associated high-dimensional effects.

    %\item \textbf{Feature-space kernels rely on an incorrect assumption on $\phi_j$.}
    \item \textbf{Linearity Assumption for Feature-Space Velocity Approximation.}
    Evaluating kernels in feature space requires introducing Flow Matching directly in the feature domain. To make the corresponding feature-space closed-form velocity (\autoref{eq:feature_space_close_form}) compatible with the velocity term in \autoref{eq:feature_discrete} (e.g., $J_{\phi_j}(x^\theta_{1-\Delta t})u_{1-\Delta t}(x^\theta_{1-\Delta t})$), the derivation typically adopts a linearity assumption on $\phi_j$, leading to a simplified relationship between feature-space and data-space dynamics (e.g., $u_t^{z,*}(z)\approx \mathbb{E}_{x\sim p_{x|z}}[J_{\phi_j}(x)u_t^{x,*}(x)]$ in \autoref{eq:feature_space_close_form})). In practice, however, feature mappings $\phi_j$ are generally nonlinear (e.g., MAE encoders), and \cite{deng2026generative} mitigates the impact of this approximation by enlarging the number of feature representations considered.
\end{enumerate}

%Consequently, the optimization direction induced by feature-space Drifting Model is only an approximation to the data-space objective, which provides an intuitive explanation for why many features (often thousands) are needed in practice: multiple feature maps jointly form a more faithful surrogate direction for optimizing $\psi_{0\to 1}^\theta$. Our experiments in \autoref{sec:result_image_generation} further support this view, showing that using only a small number of features leads to poor optimization even when the features are information-preserving.

The reliance on a linearity assumption for $\phi_j$ also provides an intuitive explanation for why many features (often thousands) are used in practice \cite{deng2026generative}: multiple feature maps jointly form a faithful surrogate direction for optimizing $\psi_{0\to 1}^\theta$. Our experiments in \autoref{sec:result_image_generation} further support this view, showing that using only a small number of features leads to weaker optimization performance even when the features are information-preserving.

This observation also leaves an open question:
\begin{question}
How can we better design $\phi_j$, or modify the feature-space objective, so that feature-space optimization more faithfully matches the original-space objective?
\end{question}

\section{Experiment}
We conduct experiments in two settings. (1) 2D examples for likelihood learning: we validate our likelihood learning framework against analytically available ground truth. (2) Image generation: we implement the method on CelebA-HQ dataset, study the effects of learning rate and batch size, and verify the key claims of feature-space optimization. Implementation-wise, our derivation differs from Drifting Model primarily in the kernel choice; we further evaluate the flow-map-induced Gaussian kernel on CelebA-HQ and demonstrate its practicality.

\subsection{2D Examples for Likelihood Learning}

We consider three 2D point datasets, Spiral, Checkerboard, and Two Moons (\autoref{fig:2D_example_1}), each with 50K samples. We learn the generative Long-Short Flow Map using \autoref{eq:second_order_loss} with an 8-layer MLP of width 128; the resulting samples are shown in \autoref{fig:2D_example_2}.  To validate our likelihood learning formulation (\autoref{eq:likelihood_loss}), we use the same MLP architecture and compare two training views, Eulerian and Lagrangian. In the Eulerian view, we set $p_{\text{sample}}$ to a uniform distribution over the domain to learn the global density. In the Lagrangian view, we sample $p_{\text{sample}}=\psi^\theta_{t\to r}(x_0)$,$x_0\sim p_0$, where $\psi^\theta_{t\to r}$ is the learned flow map, focusing density learning around the data manifold. The Eulerian and Lagrangian results are reported in \autoref{fig:2D_example_3} and \autoref{fig:2D_example_4}. Overall, our Long-Short Flow-Map likelihood recovers the data distribution well.  All runs use a batch size of 4096 and a learning rate of $10^{-4}$. We train for 600K, 50K, and 150K steps on Spiral, Checkerboard, and Two Moons, respectively, and set $\epsilon$ to 0.5, 0.2, and 0.5 accordingly.

\subsection{Image Generation}\label{sec:result_image_generation}

\begin{figure}[t]
    \centering
    \includegraphics[width=\textwidth]{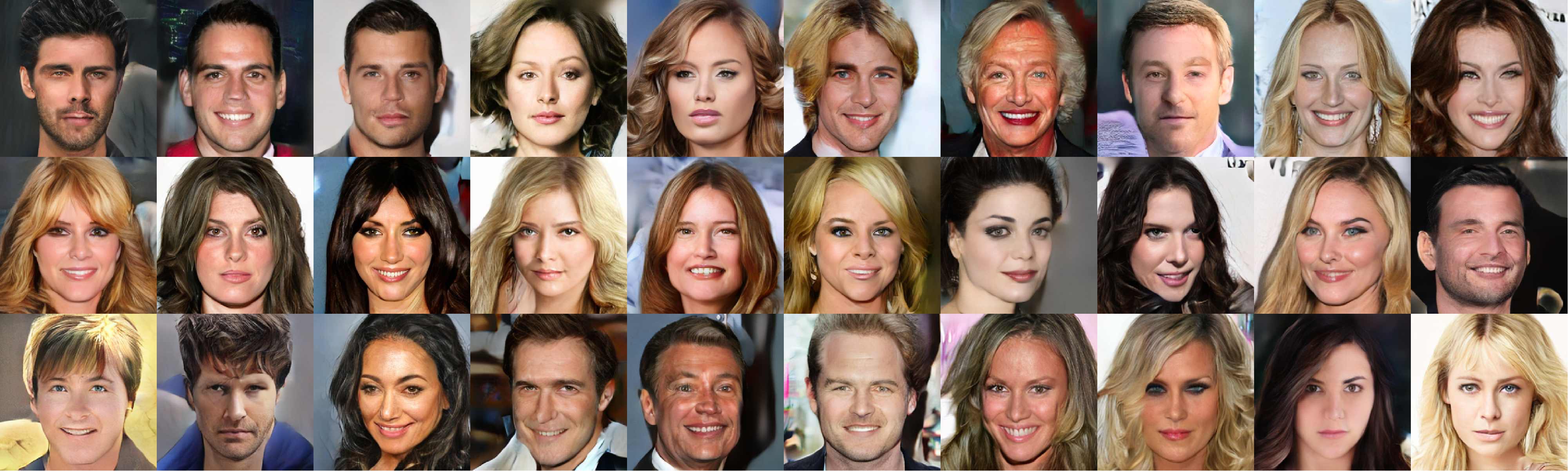}
    \caption{Unconditional latent generation results on CelebA-HQ. We train for 100K steps using feature-space optimization with a pretrained MAE as the feature extractor and achieve an FID of 14.71. For comparison, the corresponding MeanFlow model reaches an FID of 12.4 after 400K training steps. Full results are provided in \autoref{tab:bs_steps_fid}.}
    \label{fig:celeba_results}
    \vspace{-1.5ex}
\end{figure}

We evaluate the Long-Short Flow-Map method/Drifting Model on unconditional image generation using CelebA-HQ \cite{karras2017progressive}, which contains 30,000 high-quality face images. Following \cite{deng2026generative}, we perform latent space generation with a DiT-B/2 backbone \cite{peebles2023scalable} and a pretrained VAE from Stable Diffusion~\cite{rombach2022high}, and optimize the objective in feature space. Concretely, we extract latent features using a pretrained MAE encoder \cite{he2022masked}. Detailed architectures, hyperparameters, and training settings are provided in \autoref{sec:training}. We first sweep the learning rate and batch size as in \autoref{sec:image_generation_result} to study their impact on sample quality, and the results are shown in \autoref{fig:image_generation1}, \ref{fig:image_generation2} and \ref{fig:image_generation3}.

We further test the claim in \autoref{sec:feature_space} that feature-space optimization is an imperfect proxy for the original objective, and therefore benefits from using many features to better recover the true update direction. To conduct this ablation study, we compare the default setting in \cite{deng2026generative}, which uses thousands of features, with a reduced setting that uses only four features obtained by taking the full feature maps from the four-stage MAE encoder outputs at different resolutions and treating each full feature map as a single feature. We also include a baseline without feature-space optimization, as shown in \autoref{fig:image_generation6} and \autoref{fig:drifting_ablation}.  Without feature-space optimization, training fails completely. Using only four features enables learning, but introduces visible distortions compared with using thousands of features, even though these four features collectively retain the full latent information.

Finally, our derivation of the Long-Short Flow-Map method differs from the original Drifting Model primarily in the kernel form. The Drifting Model adopts a Laplacian kernel $\exp(-\|y-x\|/\epsilon)$ for $k_t(\cdot,\cdot)$, whereas our flow-map derivation with the initial Gaussian noise yields a Gaussian kernel, $\exp(-\|y-x\|^2/\epsilon)$. As shown in \autoref{fig:image_generation4} and \autoref{fig:image_generation5}, the Gaussian kernel remains effective on CelebA-HQ and achieves results comparable to, and in some settings even better than, the original kernel. Moreover, \citet{deng2026generative} reports ImageNet experiments with batch sizes of at least 4096, which raises concerns about the practical applicability of the Drifting Model. In contrast, we find that with appropriate learning-rate tuning, the long-short flow-map method/Drifting Model can still produce strong samples on CelebA-HQ with a standard batch size of 64 (see \autoref{fig:bs64_25k}), suggesting that the method may be more practical than previously implied.

\section{Related Work}\label{sec:related_work}
\paragraph{Flow Map Methods.}
Diffusion and flow matching models typically require many sampling steps, which motivates recent interest in one step or few step generation. Consistency Models \cite{song2023consistency, song2023improved, geng2025consistency, lu2025simplifying} are an early representative that explicitly learns long range mappings by enforcing consistency across time. The flow map viewpoint was later formalized in \cite{boffi2025flow}, where the core objective is to learn time indexed transport maps that satisfy trajectory consistency, also known as the semigroup property. Building on this perspective, a line of methods such as Shortcut Models \cite{frans2025shortcut} and SplitMeanFlow \cite{guo2025splitmeanflow} progressively distill an instantaneous velocity field into long horizon mappings. In parallel, MeanFlow \cite{geng2025mean}, $\alpha-$Flow \cite{zhang2025alphaflow}, and subsequent variants \cite{li2025functional, li2026trajectory, geng2025improved, lu2026one} learn via a continuity equation induced by the mean velocity defined by the flow map, providing an alternative route to train long distance transport without explicitly integrating many small steps. 
The long-short flow-map perspective was also proposed for the simulation of flow dynamics in physical space (e.g., \cite{yin2021characteristic,yin2023characteristic,nave2010gradient,zhou2024eulerian,deng2023fluid}), where adaptive flow maps are used to reduce numerical dissipation in transport and fluid simulation. Prior work in this literature suggests that a single-scale flow map can be suboptimal, and that combining a long-range map for accumulated transport with a short-range map for local updates can better balance long-term fidelity and local accuracy (e.g., see \cite{zhou2024eulerian}).

\paragraph{Closed-Form}  Different from prior work that mainly optimizes the Flow Matching regression objective, recent studies analyze the induced velocity field via the analytic closed-form solution \cite{bertrand2025closed}. Since evaluating such velocities requires expensive kernel-weighted aggregation over the dataset, this line of work is used mostly for theory, including stability analysis \cite{sprague2024stable}, memorization and generalization \cite{gao2024flow, bertrand2025closed, li2026kinetic}, and inference-time dynamics and acceleration \cite{wan2024elucidating}.

\paragraph{Distribution Moment Matching.}
Moment matching aligns model and data distributions by minimizing kernel discrepancies under the Maximum Mean Discrepancy (MMD) criterion \cite{li2015generative, dziugaite2015training}. In GANs, this idea has been widely adopted, for example in MMD GANs and related variants \cite{li2017mmd}, which match distribution moments using Gaussian or Laplace kernels. Closely related, Coulomb GAN \cite{unterthiner2017coulomb} models learning via a Coulomb-style potential field and employs a Plummer-type kernel, which has been shown to lead to the optimal solution. Recently, moment matching has regained attention in one-step generation: inductive moment matching extends MMD-style objectives to one- or few-step diffusion by constructing positive and negative samples along the sampling path \cite{zhou2025inductive}. Drifting Model \cite{deng2026generative} is also kernel-based, but instead formulates distribution alignment through a drifting field that measures the discrepancy between the current and target distributions and explicitly governs sample updates during training. Compared to standard MMD objectives, it provides a more flexible framework via a normalized kernel formulation, feature-space optimization, and extensions that support classifier-free guidance.

%\paragraph{Long-Short Flow Map}
%Our Long-Short Flow Map for one-step generation is inspired by fluid simulation methods in computational physics \cite{yin2021characteristic,yin2023characteristic,nave2010gradient} and computer graphics \cite{zhou2024eulerian,deng2023fluid}. In transport and incompressible flow simulation, a common idea is to explicitly maintain a long-range flow map (or characteristic map) from the initial time to the current time, and represent transported quantities as the composition of this map with the initial field, which helps reduce numerical dissipation compared with standard additive Eulerian updates. In computational physics, this line of work is often called characteristic mapping (CM), while in computer graphics it is commonly studied as flow-map-based simulation. Prior work has maintained and evolved flow maps using grid  \cite{sun2025leapfrog}, particle \cite{li2024lagrangian}, hybrid grid-particle \cite{zhou2024eulerian, li2024particle}, and neural representations \cite{deng2023fluid}.  These methods also suggest that a single-scale flow map is often suboptimal in practice. A more effective strategy is to decompose the full flow map into a long-range map that accumulates historical transport and a short-range map that captures local evolution near the current time  \cite{yin2021characteristic}. The complete flow map is then obtained by composition, combining long-term geometric fidelity with accurate local updates.

\section{Conclusion}

In this paper, we present the long-short flow-map perspective of the Drifting Model. This interpretation helps explain several of its key design choices. Building on this understanding, we derive a likelihood-based learning formulation and provide an in-depth discussion of feature-space optimization.

Several questions remain open. First, conditional generation with classifier-free guidance would benefit from a more principled formulation of CFG within this framework, together with systematic empirical evaluation. Second, the design of the feature-space encoder and the associated objective is not yet fully understood; in particular, it remains unclear how to construct feature representations that yield faithful optimization directions with fewer or more structured features. Looking ahead, this long-short flow-map perspective may extend beyond image generation to settings such as function generation and generation on manifolds. Finally, our analysis reveals an explicit connection between the noise distribution and the induced kernel form, suggesting a potential link between kernel design and recent noise optimization strategies \cite{huang2024blue}. We expect that further exploration of this connection could motivate new methodological developments.

\bibliographystyle{unsrtnat}
\bibliography{example_paper,3D_generation}  

\newpage
\appendix
\onecolumn
\section{Missing Proofs and Theorems}
\subsection{Proof of \autoref{prop:close_form1}}\label{sec:proof_close_form1}
\paragraph{\autoref{prop:close_form1} Closed-Form Solution of Flow Matching} \autoref{eq:FlowMatchingLoss} admits an optimal solution
\begin{equation}
u_t^*(x)
= \mathbb{E}_{x_1\sim p_{1|t}(x_1|x)}\!\big[u_t(x| x_1)\big]
= \frac{\mathbb{E}_{x_1\sim p_{1|t}(x_1|x)}[x_1]-x}{1-t},
\end{equation}
which further admits a closed-form expression
\begin{equation}
    \begin{aligned}
        u_t^*(x) &= \frac{\mathbb{E}_{x_1\sim p_1}[\frac{x_1-x}{1-t}k_{t}(x_1,x)]}{\mathbb{E}_{x_1\sim p_1}[k_{t}(x_1,x)]}, \qquad
        k_{t}(y,x) &= e^{-\frac{\|ty-x\|^2}{2(1-t)^2}}
    \end{aligned}
\end{equation}
\begin{proof}
    Because the objective in \autoref{eq:FlowMatchingLoss} is an $l_2$ regression, the minimizer is given by the conditional expectation namely, $u_t^*(x)=\mathbb{E}[u_t(x| x_1)| x] = \mathbb{E}_{x_1\sim p_{1|t}(x_1|x)}[u_t(x|x_1)]$.  Because $u_t(x|x_1) = \frac{x_1-x}{1-t}$, we have $u_t^*(x) = \mathbb{E}_{x_1\sim p_{1|t}(x_1|x)}[\frac{x_1-x}{1-t}] = \frac{\mathbb{E}_{x_1\sim p_{1|t}(x_1|x)}[x_1]-x}{1-t}$.

    Expanding the expectation $\mathbb{E}_{x_1\sim p_{1|t}(x_1|x)}[\frac{x_1-x}{1-t}]$, we obtain
    \begin{equation}
        \begin{aligned}
            \mathbb{E}_{x_1\sim p_{1|t}(x_1|x)}[\frac{x_1-x}{1-t}] &= \int \frac{x_1-x}{1-t} p_{1|t}(x_1|x) dx_1\\
            &= \int \frac{x_1-x}{1-t} \frac{p_{t|1}(x|x_1)p_1(x_1)}{\int p_{t|1}(x|x_1)p_1(x_1)dx} dx_1\\
            &= \frac{\int \frac{x_1-x}{1-t} p_{t|1}(x|x_1)p_1(x_1) dx_1}{\int p_{t|1}(x|x_1)p_1(x_1)dx}\\
        \end{aligned}
    \end{equation}
    We next compute $\int p_{t|1}(x| x_1)p_1(x_1) dx_1$ and $\int \frac{x_1-x}{1-t} p_{t|1}(x| x_1)p_1(x_1) dx_1$.  For $\int \frac{x_1-x}{1-t} p_{t|1}(x| x_1)p_1(x_1) dx_1$, we have
    \begin{equation}
        \begin{aligned}
            \int \frac{x_1-x}{1-t} p_{t|1}(x|x_1)p_1(x_1)dx &= \int \int \frac{x_1-x}{1-t} p_{t|1,0}(x|x_1,x_0)p_0(x_0)dx_0p_1(x_1)dx\\
            &= \int \int \frac{x_1-x}{1-t} \delta _{x = tx_1 + (1-t)x_0} p_0(x_0)dx_0p_1(x_1)dx\\
            &= \int \int \frac{x_1-x}{1-t} \delta _{x = tx_1 + (1-t)x_0} p_0(\frac{x-tx_1}{1-t})dx_0p_1(x_1)dx\\
            &= \int \frac{x_1-x}{1-t} p_0(\frac{x-tx_1}{1-t})p_1(x_1)dx\\
            &= \int \frac{x_1-x}{1-t} \frac{1}{(2\pi)^\frac{d}{2}}e^{-\|\frac{x-tx_1}{1-t}\|^2}p_1(x_1)dx\\
            &=  \frac{1}{(2\pi)^\frac{d}{2}}\mathbb{E}_{x_1\sim p_1}[\frac{x_1-x}{1-t} k_t(x_1,x)]
        \end{aligned}
    \end{equation}
    Similarly, we have $\int  p_{t|1}(x|x_1)p_1(x_1)dx = \frac{1}{(2\pi)^\frac{d}{2}}\mathbb{E}_{x_1\sim p_1}[k_t(x_1,x)]$.  Therefore, we obtain:
    \begin{equation}
        \begin{aligned}
            u_t^*(x) &=\frac{\mathbb{E}_{x_1\sim p_1}[\frac{x_1-x}{1-t}k_{t}(x_1,x)]}{\mathbb{E}_{x_1\sim p_1}[k_{t}(x_1,x)]}
        \end{aligned}
    \end{equation}
\end{proof}
\subsection{Non-Existence of Conditional Flow Maps}
\begin{theorem}[Non-Existence of Conditional Flow Maps \cite{li2026trajectory}]\label{thm:non_existence_conditioanl}
     There exists no conditional flow maps $\psi_{t\to r}(x | x_1)$ that simultaneously (1) is consistent with the conditional velocity $u(x|x_1)$ under \autoref{eq:evolv_flow_map}, and (2) satisfies the consistency relation $\psi_{t\to r}(x)=\mathbb{E}_{x_{1}\sim p_{1|t}(x_1|x)}[\psi_{t\to r}(x|x_1)]$ with marginal flow maps. As a result, a self-consistent conditional flow map does not exist. 
\end{theorem}
\begin{proof}
First, we denote the mappings $\psi_{t\to r}(x)$ obtained from (1) and (2) as $\psi^{(1)}_{t\to r}(x)$ and $\psi^{(2)}_{t\to r}(x)$, respectively. Specifically, $\psi^{(1)}_{t\to r}(x)=\psi_{0\to r}(\psi_{0\to t}^{-1}(x))$, and $\psi^{(2)}_{t\to r}(x)=\mathbb{E}_{x_1\sim p_t(x_1|x)}[\psi_{t\to r}(x|x_1)]$. It suffices to show that $\psi^{(1)}_{t\to r}(x)\neq \psi^{(2)}_{t\to r}(x)$. To this end, it is sufficient to prove that $\frac{d}{dt}\psi^{(1)}_{t\to r}(x)\neq \frac{d}{dt}\psi^{(2)}_{t\to r}(x)$ at $t=0$.  

\begin{equation}
    \begin{aligned}
        \frac{d}{dr}\psi^{(2)}_{t\to r}(x)&= \frac{d}{dr}\int_{x_1} \psi_{t\to r}(x|x_1)p_{1|t}(x_1|x)dx_1\\
        &=\int_{x_1} \frac{d}{dr}\psi_{t\to r}(x|x_1)p_{1|t}(x_1|x)dx_1\\
        &=\int_{x_1} u_{r}(\psi_{t\to r}(x|x_1)|x_1)\frac{p_1(x_1)p_{t|1}(x|x_1)}{p_t(x)}dx_1\\
        &=\int_{x_1} u_{r}(\psi_{t\to r}(x|x_1)|x_1)p_1(x_1)dx_1\\    
        &=\int_{x_1}(x_1-x)p_1(x_1)dx_1\\    
        &= \mathbb{E}_{x_1\sim p_{1}(x_1)}[x_1] - x
    \end{aligned}
\end{equation}
Consequently, if $\psi^{(1)}_{t\to r}(x) = \psi^{(2)}_{t\to r}(x)$ we must have $\frac{d}{dr}\psi_{0\to r}(x) = \mathbb{E}_{x_1\sim p_{1}(x_1)}[x_1] - x$, $\psi_{0\to r}(x) =(\mathbb{E}_{x_1\sim p_{1}(x_1)}[x_1]-x) r + x$, which implies $p_{1} = (\psi_{0\to 1})_\sharp p_0 = \delta_{\mathbb{E}_{x_1\sim p_{1}(x_1)}[x_1]}$, where $\delta$ denotes the Dirac distribution at a single point, which contradicts the actual setting $p_1 = p_{data}$.
\end{proof}
\subsection{Proof of \autoref{thm:close_form2}}\label{sec:proof_close_form2}
\paragraph{\autoref{thm:close_form2} Closed-Form Solution of the Endpoint Velocity} For Flow Matching, under mild smoothness assumptions on the underlying distributions, the optimal velocity satisfies $\lim_{t\to 1^-} u_t^*(x)=x$. Consequently, the left-limit endpoint velocity $u_{1^-}^*(x)\triangleq \lim_{t\to 1^-} u_t^*(x)$ admits the closed form for any $t\in [0,1)$
\begin{equation}
    \begin{aligned}
        u^*_{1^-}(x) 
        = \mathbb{E}_{x_t \sim p_{t|1}(x_t|x)}[ \frac{x-x_t}{1-t}] = \frac{\mathbb{E}_{x_t\sim p_t}[\frac{x-x_t}{1-t}k_t(x,x_t)]}{\mathbb{E}_{x_t\sim p_t}[k_t(x,x_t)]}
    \end{aligned}
\end{equation}
\begin{proof}
In the derivations below, for convenience, we adopt the random-variable notation from Flow Matching and use the conditional path $X_t = tX_1 + (1-t)X_0$, where $X_1\sim p_1$, $X_0\sim p_0$, and $X_t\sim p_t$ are the random variables induced by this conditional probability path. 
For sufficiently smooth distributions, the limit $\lim_{\epsilon\to 0^+}\mathbb{E}[X_0 | X_{1-\epsilon}=x] = \mathbb{E}[X_0 | X_1=x]$ exists. Therefore,
\begin{equation}
    \begin{aligned}
        \lim_{t\to 1^-} u_t(x) 
        &= \lim_{\epsilon\to 0^+}\frac{x- \mathbb{E}[X_0 | X_{1-\epsilon}=x]}{1-(1-\epsilon)}\\
        &= x-\mathbb{E}[X_0 | X_{1}=x] \\
        &\overset{\textcircled{1}}{=} x-\mathbb{E}[X_0] \\
        &= x,
    \end{aligned}
\end{equation}
where \textcircled{1} uses the independence of $X_0$ and $X_1$.  We next show that $\mathbb{E}[X_t| X_1=x]=tx$.
\begin{equation}
    \begin{aligned}
        \mathbb{E}[X_t|X_1=x] &= \mathbb{E}[tX_1 + (1-t)X_0|X_1=x] \\
        &= tx + (1-t)\mathbb{E}[X_0|X_1=x]\\
        &= tx + (1-t)\mathbb{E}[X_0]\\
        &= tx
    \end{aligned}
\end{equation}    
Consequently, the endpoint velocity admits the representation
\begin{equation}
    \begin{aligned}
        u_{1^-}(·x) &= \lim_{t\to 1} u_t(x) \\
        &= x\\
        &= \frac{x-\mathbb{E}[X_t|X_1=x]}{1-t}\\
        &= \mathbb{E}_{x_t \sim p_{t|1}(x_t|x)}[ \frac{x-x_t}{1-t}]
    \end{aligned}
\end{equation}
Similar to \autoref{sec:proof_close_form1}, we now derive a closed-form expression for $u_1(x)$:
\begin{equation}
    \begin{aligned}
        u_{1^-}(x) &= \mathbb{E}_{x_t \sim p_{t|1}(x_t|x)}[ \frac{x-x_t}{1-t}]\\
        &= \int \frac{x-x_t}{1-t} p_{t|1}(x_t|x) d x_t\\
        &= \int \frac{x-x_t}{1-t} \frac{p_{1|t}(x|x_t)p_t(x_t)}{\int _{x_t}p_{1|t}(x|x_t)p_t(x_t) dx_t} d x_t\\
        &=  \frac{\int \frac{x-x_t}{1-t} p_{1|t}(x|x_t)p_t(x_t)d x_t}{\int p_{1|t}(x|x_t)p_t(x_t) dx_t} 
    \end{aligned}
\end{equation}
Next, we compute $\int p_{1|t}(x|x_t)p_t(x_t) dx_t$ and $\int \frac{x-x_t}{1-t} p_{1|t}(x|x_t)p_t(x_t)d x_t$.  For $\int \frac{x-x_t}{1-t} p_{1|t}(x|x_t)p_t(x_t)d x_t$, we have
\begin{equation}
    \begin{aligned}
        \int\frac{x-x_t}{1-t}p_{1|t}(x|x_t)p_t(x_t) dx_t &=  \int_{x_t}\int_{x_0}\frac{x-x_t}{1-t}p_{1|t,0}(x|x_t,x_0)p_0(x_0)dx_0p_t(x_t) dx_t\\
        &= \int\frac{x-x_t}{1-t}\int\delta(x=\frac{x_t -(1-t)x_0}{t})\frac{1}{(2\pi)^\frac{d}{2}}e^{-\frac{1}{2}x_0^2}dx_0p_t(x_t) dx_t\\
        &= \int\frac{x-x_t}{1-t}\frac{1}{(2\pi)^\frac{d}{2}}e^{-\frac{1}{2}\|\frac{x_t-tx}{1-t}\|^2}p_t(x_t) dx_t
    \end{aligned}
\end{equation}
Similarly, we have $\int p(x|x_t)p(x_t) dx_t = \int \frac{1}{(2\pi)^\frac{d}{2}}e^{-\frac{1}{2}\|\frac{x_t-tx}{1-t}\|^2}p(x_t) dx_t$.  Therefore, we have
\begin{equation}
    \begin{aligned}
        u_{1^-}(x) = \frac{\mathbb{E}_{x_t\sim p_t}[\frac{x-x_t}{1-t}k_t(x,x_t)]}{\mathbb{E}_{x_t\sim p_t}[k_t(x,x_t)]}
    \end{aligned}
\end{equation}
\end{proof}

\subsection{Proof of \autoref{thm:close_form_divergence}}\label{sec:proof_close_form_divergence}
\paragraph{\autoref{thm:close_form_divergence} Closed-Form Solution for Velocity Divergence}
For the closed-form solution of the velocity field $u^*_t(x)$ in \autoref{eq:close_form1} and the endpoint velocity $u^*_{1^-}(x)$ in \autoref{eq:close_form2} with the mollified kernel $k_t^\epsilon$ in \autoref{eq:mollified_kernel}, their divergences can be also computed in closed form as
\begin{equation}
    \begin{aligned}
        \nabla \cdot u^*_t(x) &= \frac{1}{1-t}[\frac{2t}{2(1-t)^2+\epsilon}(\frac{\mathbb{E}_{x_1\sim p_1}[\|x_1\|^2k^\epsilon_{t}(x_1,x)]}{\mathbb{E}_{x_1\sim p_1}[k^\epsilon_{t}(x_1,x)]}\\
        &-\|\frac{\mathbb{E}_{x_1\sim p_1}[x_1k^\epsilon_{t}(x_1,x)]}{\mathbb{E}_{x_1\sim p_1}[k^\epsilon_{t}(x_1,x)]}\|^2)-d]\\
        \nabla\cdot u^*_{1^-}(x) &= \frac{1}{1-t}[d-\frac{2t^2}{2(1-t)^2+\epsilon}(\frac{\mathbb{E}_{x_t\sim p_t}[\|x_t\|^2k^\epsilon_{t}(x,x_t)]}{\mathbb{E}_{x_t\sim p_t}[k^\epsilon_{t}(x,x_t)]}\\
        &-\|\frac{\mathbb{E}_{x_t\sim p_t}[x_tk^\epsilon_{t}(x,x_t)]}{\mathbb{E}_{x_t\sim p_t}[k^\epsilon_{t}(x,x_t)]}\|^2)]
    \end{aligned}
\end{equation}
\begin{proof}
    First, we compute the partial derivative of $k^\epsilon_t(y,x)$:
    \begin{equation}
        \begin{aligned}
           \frac{\partial k^\epsilon_t(y,x)}{\partial y} =  -\frac{2t(ty-x)}{2(1-t)^2+\epsilon} e^{-\frac{\|ty-x\|^2}{2(1-t)^2+\epsilon}}\\
           \frac{\partial k^\epsilon_t(y,x)}{\partial x} =  \frac{2(ty-x)}{2(1-t)^2+\epsilon} e^{-\frac{\|ty-x\|^2}{2(1-t)^2+\epsilon}}\\
        \end{aligned}
    \end{equation}
    Then, we calculate $\nabla \cdot u^*_t(x)$ and $\nabla\cdot u^*_1(x)$, respectively.  For $\nabla \cdot u^*_t(x)$, we have
    \begin{equation}
        \begin{aligned}
            &\nabla\cdot u_t^*(x)\\
            &= \nabla\cdot \Big(\frac{\mathbb{E}_{x_1\sim p_1}[\frac{x_1-x}{1-t}k_t^\epsilon(x_1,x)]}{\mathbb{E}_{x_1\sim p_1}[k_t^\epsilon(x_1,x)]}\Big)\\
            &= \frac{\nabla\cdot \Big(\frac{\mathbb{E}_{x_1\sim p_1}[x_1k_t^\epsilon(x_1,x)]}{\mathbb{E}_{x_1\sim p_1}[k_t^\epsilon(x_1,x)]}\Big) -d}{1-t}\\
            &= \frac{1}{1-t}[\frac{\nabla\cdot(\mathbb{E}_{x_1\sim p_1}[x_1k_t^\epsilon(x_1,x)])\mathbb{E}_{x_1\sim p_1}[k_t^\epsilon(x_1,x)]-\mathbb{E}_{x_1\sim p_1}[x_1k_t^\epsilon(x_1,x)]\cdot\nabla(\mathbb{E}_{x_1\sim p_1}[k_t^\epsilon(x_1,x)])}{(\mathbb{E}_{x_1\sim p_1}[k_t^\epsilon(x_1,x)])^2} -d]\\
            &= \frac{1}{1-t}[\frac{\mathbb{E}_{x_1\sim p_1}[x_1 \cdot \nabla k_t^\epsilon(x_1,x)]\mathbb{E}_{x_1\sim p_1}[k_t^\epsilon(x_1,x)]-\mathbb{E}_{x_1\sim p_1}[x_1k_t^\epsilon(x_1,x)]\cdot\mathbb{E}_{x_1\sim p_1}[\nabla k_t^\epsilon(x_1,x)]}{(\mathbb{E}_{x_1\sim p_1}[k_t^\epsilon(x_1,x)])^2} -d]\\            
            &= \frac{1}{1-t}[\frac{\mathbb{E}_{x_1\sim p_1}[x_1\cdot\frac{2(tx_1-x)}{2(1-t)^2+\epsilon}k_t^\epsilon(x_1,x)]\mathbb{E}_{x_1\sim p_1}[k_t^\epsilon(x_1,x)]-\mathbb{E}_{x_1\sim p_1}[x_1k_t^\epsilon(x_1,x)]\cdot\mathbb{E}_{x_1\sim p_1}[\frac{2(tx_1-x)}{2(1-t)^2+\epsilon}k_t^\epsilon(x_1,x)]}{(\mathbb{E}_{x_1\sim p_1}[k_t^\epsilon(x_1,x)])^2} -d]\\        
            &= \frac{1}{1-t}[\frac{2t}{2(1-t)^2+\epsilon}\frac{\mathbb{E}_{x_1\sim p_1}[x_1\cdot x_1 k_t^\epsilon(x_1,x)]\mathbb{E}_{x_1\sim p_1}[k_t^\epsilon(x_1,x)]-\mathbb{E}_{x_1\sim p_1}[x_1k_t^\epsilon(x_1,x)]\cdot\mathbb{E}_{x_1\sim p_1}[x_1 k_t^\epsilon(x_1,x)]}{(\mathbb{E}_{x_1\sim p_1}[k_t^\epsilon(x_1,x)])^2} -d]\\    
            &= \frac{1}{1-t}[\frac{2t}{2(1-t)^2+\epsilon}(\frac{\mathbb{E}_{x_1\sim p_1}[\|x_1\|^2k^\epsilon_{t}(x_1,x)]}{\mathbb{E}_{x_1\sim p_1}[k^\epsilon_{t}(x_1,x)]}-\|\frac{\mathbb{E}_{x_1\sim p_1}[x_1k^\epsilon_{t}(x_1,x)]}{\mathbb{E}_{x_1\sim p_1}[k^\epsilon_{t}(x_1,x)]}\|^2)-d]
        \end{aligned}
    \end{equation}
    For $\nabla \cdot u^*_{1^-}(x)$, we have
    \begin{equation}
        \begin{aligned}
            &\nabla \cdot u^*_{1^-}(x)\\
            &= \nabla \cdot\Big(\frac{\mathbb{E}_{x_t\sim p_t}[\frac{x-x_t}{1-t}k_t^\epsilon(x,x_t)]}{\mathbb{E}_{x_t\sim p_t}[k_t^\epsilon(x,x_t)]}\Big)\\
            &=   \frac{1}{1-t}[d -  \frac{\nabla\cdot(\mathbb{E}_{x_t\sim p_t}[x_t k_t^\epsilon(x,x_t)])\mathbb{E}_{x_t\sim p_t}[k_t^\epsilon(x,x_t)]-\mathbb{E}_{x_t\sim p_t}[x_t k_t^\epsilon(x,x_t)]\cdot\nabla(\mathbb{E}_{x_t\sim p_t}[k_t^\epsilon(x,x_t)])}{(\mathbb{E}_{x_t\sim p_t}[k_t^\epsilon(x,x_t)])^2}]\\
            &=   \frac{1}{1-t}[d -  \frac{\mathbb{E}_{x_t\sim p_t}[x_t \cdot \nabla k_t^\epsilon(x,x_t)]\mathbb{E}_{x_t\sim p_t}[k_t^\epsilon(x,x_t)]-\mathbb{E}_{x_t\sim p_t}[x_t k_t^\epsilon(x,x_t)]\cdot(\mathbb{E}_{x_t\sim p_t}[\nabla k_t^\epsilon(x,x_t)])}{(\mathbb{E}_{x_t\sim p_t}[k_t^\epsilon(x,x_t)])^2}]\\
            &=   \frac{1}{1-t}[d -  \frac{\mathbb{E}_{x_t\sim p_t}[x_t \cdot \frac{2(tx-x_t)}{2(1-t)^2+\epsilon} k_t^\epsilon(x,x_t)]\mathbb{E}_{x_t\sim p_t}[k_t^\epsilon(x,x_t)]-\mathbb{E}_{x_t\sim p_t}[x_t k_t^\epsilon(x,x_t)]\cdot(\mathbb{E}_{x_t\sim p_t}[\frac{2(tx-x_t)}{2(1-t)^2+\epsilon} k_t^\epsilon(x,x_t)])}{(\mathbb{E}_{x_t\sim p_t}[k_t^\epsilon(x,x_t)])^2}]\\
            &=   \frac{1}{1-t}[d -  \frac{\mathbb{E}_{x_t\sim p_t}[x_t \cdot \frac{-2t(tx-x_t)}{2(1-t)^2+\epsilon} k_t^\epsilon(x,x_t)]\mathbb{E}_{x_t\sim p_t}[k_t^\epsilon(x,x_t)]-\mathbb{E}_{x_t\sim p_t}[x_t k_t^\epsilon(x,x_t)]\cdot(\mathbb{E}_{x_t\sim p_t}[\frac{-2t(tx-x_t)}{2(1-t)^2+\epsilon} k_t^\epsilon(x,x_t)])}{(\mathbb{E}_{x_t\sim p_t}[k_t^\epsilon(x,x_t)])^2}]\\
            &=\frac{1}{1-t}[d-\frac{2t^2}{2(1-t)^2+\epsilon}(\frac{\mathbb{E}_{x_t\sim p_t}[\|x_t\|^2k^\epsilon_{t}(x,x_t)]}{\mathbb{E}_{x_t\sim p_t}[k^\epsilon_{t}(x,x_t)]}-\|\frac{\mathbb{E}_{x_t\sim p_t}[x_tk^\epsilon_{t}(x,x_t)]}{\mathbb{E}_{x_t\sim p_t}[k^\epsilon_{t}(x,x_t)]}\|^2)]
        \end{aligned}
    \end{equation}
\end{proof}

\subsection{Proof of \autoref{thm:feature_data_optimal}}\label{sec:proof_feature_data_optimal}
\paragraph{\autoref{thm:feature_data_optimal} Feature-Data Space Optimal Velocity Relation} Let $X_1\sim p_1$ and define $Z_1=\phi_j(X_1)$. Denote by $u_t^{x,*}(x)$ the optimal Flow Matching velocity in data space associated with the conditional path $X_t=tX_1+(1-t)X_0$, and by $u_t^{z,*}(z)$ the optimal velocity in feature space associated with $Z_t=tZ_1+(1-t)Z_0$, where $X_0$ and $Z_0$ are Gaussian reference variables on $\mathcal{X}$ and $\mathcal{Z}$, respectively.  We define $p_{x|z}(x|z)$ as the conditional distribution induced by the inverse of the deterministic map $z=\phi_j(x)$.  Then, for $z=\phi_j(x)$, the two velocities satisfy the following approximate relations:
\begin{equation}
\begin{aligned}
u_t^{z,*}(z) &\approx J_{\phi_j}(x)\frac{\mathbb{E}[X_1| Z_t=\phi_j(x)]-x}{1-t},\\
u_{1^-}^{z,*}(z) &\approx J_{\phi_j}(x)\frac{x-\mathbb{E}[X'_t|Z_1=\phi_j(x)]}{1-t},
\end{aligned}
\end{equation}
where $J_{\phi_j}(x)=\frac{\partial \phi_j(x)}{\partial x}$ is the Jacobian of $\phi_j$, and $X'_t$ is an auxiliary $\mathcal X$-valued random variable satisfying $\phi_j(X'_t)=Z_t$ almost surely.  The approximation $\approx$ above follows from the first-order Taylor expansion
$\phi_j(y)-\phi_j(x)=J_{\phi_j}(x)(y-x)+O(\|y-x\|^2)$. Only in the special case where $\phi_j(X_t)=Z_t$ holds (for instance, when $\phi_j$ is linear and the feature-space path $Z_t$ is induced by pushing forward the data-space path $X_t$), the optimal feature-space velocity can be written as 
\begin{equation}
    \begin{aligned}
 u_t^{z,*}(z) &\approx \mathbb{E}_{x\sim p_{x|z}}[ J_{\phi_j}(x)u_t^{x,*}(x)]     \\
 u_{1^-}^{z,*}(z) &\approx \mathbb{E}_{x\sim p_{x|z}} [J_{\phi_j}(x)u_{1^-}^{x,*}(x)]
    \end{aligned}
\end{equation}

\begin{proof}
    Let $z=\phi_j(x)$. The optimal velocity $u_t^{z,*}(z)$ can be computed with first-order approximation:
\begin{equation}
    \begin{aligned}
        u_t^{z,*}(z)&= \frac{\mathbb{E}[Z_1  - z|Z_t = z]}{1-t}\\
        &\overset{\textcircled{1}}{=} \frac{\mathbb{E}[\phi_j(X_1)  - \phi_j(x)|Z_t = z]}{1-t}\\
        &\approx \frac{\mathbb{E}[J_{\phi_j}(x)(X_1-x)|Z_t = z]}{1-t}\\
        &= \frac{J_{\phi_j}(x)(\mathbb{E}[X_1|Z_t = z]-x)}{1-t}
    \end{aligned}
\end{equation}
where $\textcircled{1}$ holds because 
\begin{equation}
    \begin{aligned}
        \mathbb{E}[Z_1 |Z_t = z] &=\int z_1 p_{Z_1|Z_t}(z_1|z) dz_1 \\
        &= \int \int \phi(x_1) p_{x|z}(x_1|z_1)dx_1 p_{Z_1|Z_t}(z_1|z) dz_1\\
        &= \int \phi(x_1) \int p_{x|z}(x_1|z_1)p_{Z_1|Z_t}(z_1|z) dz_1 dx_1 \\
        &= \int \phi(x_1) \int p_{X_1|Z_1,Z_t}(x_1|z_1,z)p_{Z_1|Z_t}(z_1|z) dz_1 dx_1 \\
        &= \int \phi(x_1) p_{X_1|Z_t}(x_1|z) dx_1 \\      &=  \mathbb{E}[\phi_j(X_1) |Z_t = z]
    \end{aligned}
\end{equation}
When $\phi_j(X_t) = Z_t$ holds, we have
\begin{equation}
    \begin{aligned}
        u_t^{z,*}(z)&=\int \frac{z_1-z}{1-t} p_{Z_1|Z_t}(z_1|z) dz_1 \\
        &=\int \frac{z_1-z}{1-t} \int p_{Z_1|Z_t,X_t}(z_1|z, x)p_{x|z}(x|z)dx dz_1 \\
        &\overset{Z_t = \phi_j(X_t)}{=}\int\int  \frac{z_1-\phi_j(x)}{1-t} p_{Z_1|X_t}(z_1|x)p_{x|z}(x|z)dx dz_1 \\
        &=\int  \frac{\int z_1 p_{Z_1|X_t}(z_1|x) dz_1-\phi_j(x)}{1-t} p_{x|z}(x|z)dx\\
        &\overset{\textcircled{1}}{=}\int  \frac{\int \phi_j(x_1) p_{X_1|X_t}(x_1|x)dx_1-\phi_j(x)}{1-t} p_{x|z}(x|z)dx\\
        &=\int \int\frac{\phi_j(x_1)-\phi_j(x) }{1-t}p_{X_1|X_t}(x_1|x)dx_1 p_{x|z}(x|z)dx\\
        &\approx \int \int J_{\phi_j}(x) \frac{x_1-x}{1-t}p_{X_1|X_t}(x_1|x)dx_1 p_{x|z}(x|z)dx\\
        & =  \int J_{\phi_j}(x) u_t^{x,*}(x) p_{x|z}(x|z)dx\\
        & =  \mathbb{E}_{x\sim p_{x|z}}[ J_{\phi_j}(x)u_t^{x,*}(x)]
    \end{aligned}
\end{equation}
where $\textcircled{1}$ holds because 
\begin{equation}
    \begin{aligned}
        \int z_1 p_{Z_1|X_t}(z_1|x) dz_1&=\int \int\phi_j(x_1) p_{Z_1|X_t,X_1}(z_1|x,x_1)p_{X_1|X_t}(x_1|x)dx_1dz_1\\
        &=\int \int\phi_j(x_1) \delta_{z_1 =\phi_j(x_1)}p_{X_1|X_t}(x_1|x)dx_1dz_1\\
        &=\int \phi_j(x_1) \int\delta_{z_1 =\phi_j(x_1)}dz_1p_{X_1|X_t}(x_1|x)dx_1\\
        &=\int \phi_j(x_1) p_{X_1|X_t}(x_1|x)dx_1\\
    \end{aligned}
\end{equation}
For $u_1^{z,*}(z)$, when there exists an auxiliary random variable $X'_t$ such that $\phi_j(X'_t)=Z_t$ almost surely, we have
\begin{equation}
    \begin{aligned}
        u_1^{z,*}(z) &= \frac{\mathbb{E}[z-Z_t|Z_1=z]}{1-t}\\
        &\overset{\textcircled{1}}{=} \frac{\mathbb{E}[\phi_j(x)-\phi_j(X'_t)|Z_1=z]}{1-t}\\
        &\approx \frac{\mathbb{E}[J_{\phi_j}(x-X'_t)|Z_1=z]}{1-t}\\
        &= J_{\phi_j}(x)\frac{x-\mathbb{E}[X'_t|Z_1=\phi_j(x)]}{1-t}\\
    \end{aligned}
\end{equation}
where $\textcircled{1}$ holds because 
\begin{equation}
    \begin{aligned}
        \mathbb{E}[Z_t |Z_1 = z] &=\int z_t p_{Z_t|Z_1}(z_t|z) dz_t \\
        &= \int \int \phi_j(x'_t) p_{X_t'|Z_t}(x'_t|z_t)dx'_t p_{Z_t|Z_1}(z_t|z) dz_t\\
        &= \int \phi_j(x'_t)(\int  p_{X_t'|Z_t}(x'_t|z_t) p_{Z_t|Z_1}(z_t|z) dz_t)dx'_t\\
        &= \int \phi_j(x'_t)(\int  p_{X'_t|Z_t,Z_1}(x'_t|z_t,z) p_{Z_t|Z_1}(z_t|z) dz_t)dx'_t\\
        &= \int \phi_j(x'_t)p_{X_t'|Z_1}(x_t'|z)dx'_t\\
        &= \mathbb{E}[\phi_j(X'_t) |Z_1 = z] 
    \end{aligned}
\end{equation}
When $\phi_j(X_t) = Z_t$, $X_t$ can be used as the $X_t'$, namely $\phi_j(X_t) = Z_t$.
\begin{equation}
    \begin{aligned}
         u_1^{z,*}(z) &= \int \frac{z-z_t}{1-t}p_{Z_t|Z_1}(z_t|z)dz_t\\
         &= \int \frac{z-z_t}{1-t}\int p_{Z_t|Z_1,X_1}(z_t|z,x)p_{x|z}(x|z)dxdz_t\\
         &= \int \int \frac{\phi_j(x)-z_t}{1-t}p_{Z_t| X_1}(z_t|x)p_{x|z}(x|z)dxdz_t\\
         &= \int \frac{\phi_j(x)-\int z_t p_{Z_t| X_1}(z_t|x)dz_t}{1-t}p_{x|z}(x|z)dx\\
        &\overset{\textcircled{1}}{=} \int \frac{\phi_j(x)-\int \phi_j(x_t) p_{X_t|X_1}(x_t|x)dx_t}{1-t}p_{x|z}(x|z)dx\\
        & = \int \int \frac{\phi_j(x)-\phi_j(x_t)}{1-t}p_{x|z}(x|z)p_{X_t|X_1}(x_t|x)dx_t dx\\
        &\approx  \int \int J_{\phi_j}(x)\frac{x-x_t}{1-t}p_{x|z}(x|z)p_{X_t|X_1}(x_t|x)dx_t dx\\
        & = \int  J_{\phi_j}(x)\int \frac{x-x_t}{1-t}p_{X_t|X_1}(x_t|x)dx_t p_{x|z}(x|z) dx\\
        & = \int  J_{\phi_j}(x) u_1^{x,*}(x) p_{x|z}(x|z) dx\\
        & = \mathbb{E}_{x\sim p_{x|z}}[J_{\phi_j}(x) u_1^{x,*}(x)]
    \end{aligned}
\end{equation}
where $\textcircled{1}$ holds because 
\begin{equation}
    \begin{aligned}
        \int z_t p_{Z_t| X_1}(z_t|x)dz_t&=\int\int \phi_j(x_t) p_{Z_t| X_1, X_t}(z_t|x,x_t)p_{X_t|X_1}(x_t|x)dx_tdz_t\\
        &\overset{Z_t = \phi_j(X_t)}{=}\int\int \phi_j(x_t) \delta_{z_t = \phi_j(x_t)}p_{X_t|X_1}(x_t|x)dx_tdz_t\\
        &=\int \phi_j(x_t) \int\delta_{z_t = \phi_j(x_t)}dz_tp_{X_t|X_1}(x_t|x)dx_t\\
        &=\int \phi_j(x_t) p_{X_t|X_1}(x_t|x)dx_t\\
    \end{aligned}
\end{equation}
\end{proof}

\section{Training and Results}

\subsection{Training Details}\label{sec:training}
For the image generation task, we use DiT-B/2 as the backbone. Unless otherwise specified, we train with batch sizes in \{64, 256, 1024\}, for 100K steps, using EMA with decay 0.9999 and Adam with $(\beta_1,\beta_2)=(0.9,0.95)$. We consider learning rates in \{2.5$\times$10$^{-5}$, 5$\times$10$^{-5}$, 1$\times$10$^{-4}$, 2$\times$10$^{-4}$\} and weight decay 0, with a 5K-step learning-rate warmup. Unless otherwise specified, we use a constant learning-rate schedule. For the experiment in \autoref{fig:bs64_25k}, we focus on improving generation quality under the small batch size (64) setting by training for 400K steps with a staged schedule: 5$\times$10$^{-5}$ for steps 0-200K, 2.5$\times$10$^{-5}$ for steps 200K-300K, and 1.25$\times$10$^{-5}$ for steps 300K-400K. We use FP16 mixed precision for the main model computation, while computing the kernel $k_1^\epsilon$ in \autoref{eq:forward_euler_loss} and \autoref{eq:second_order_loss} in FP32.

We use the pretrained VAE \texttt{sd-vae-ft-mse} from \cite{rombach2022high}. For the MAE encoder, we follow the ResNet-based architecture in \cite{deng2026generative} with four resolutions $(32^2,16^2,8^2,4^2)$ and a base width of 256. We train the MAE for 250K steps with EMA decay 0.9995 and Adam with $(\beta_1,\beta_2)=(0.9,0.999)$, using a learning rate of 5$\times$10$^{-4}$ without warmup. We mask $2\times2$ patches by zeroing, where each patch is independently masked with probability 0.5. When extracting MAE features, we use all feature types (a-d) from Section A.5 of \cite{deng2026generative}, and combine computations under three temperatures (0.02, 0.05, and 0.2) within the same feature space. Moreover, instead of taking features every two blocks, we take the output features at each resolution, resulting in several thousand feature channels in total.

\subsection{Results}
\subsubsection{2D Examples for Likelihood Learning}
In \autoref{fig:2D_example}, we present the ground-truth distributions, generated samples, and the learned likelihoods under both the Eulerian and Lagrangian views for the Spiral, Checkerboard, and Two Moons datasets.

\begin{figure}[t]
  \centering
  \begin{subfigure}[t]{0.24\textwidth}
    \centering
    \includegraphics[width=\linewidth]{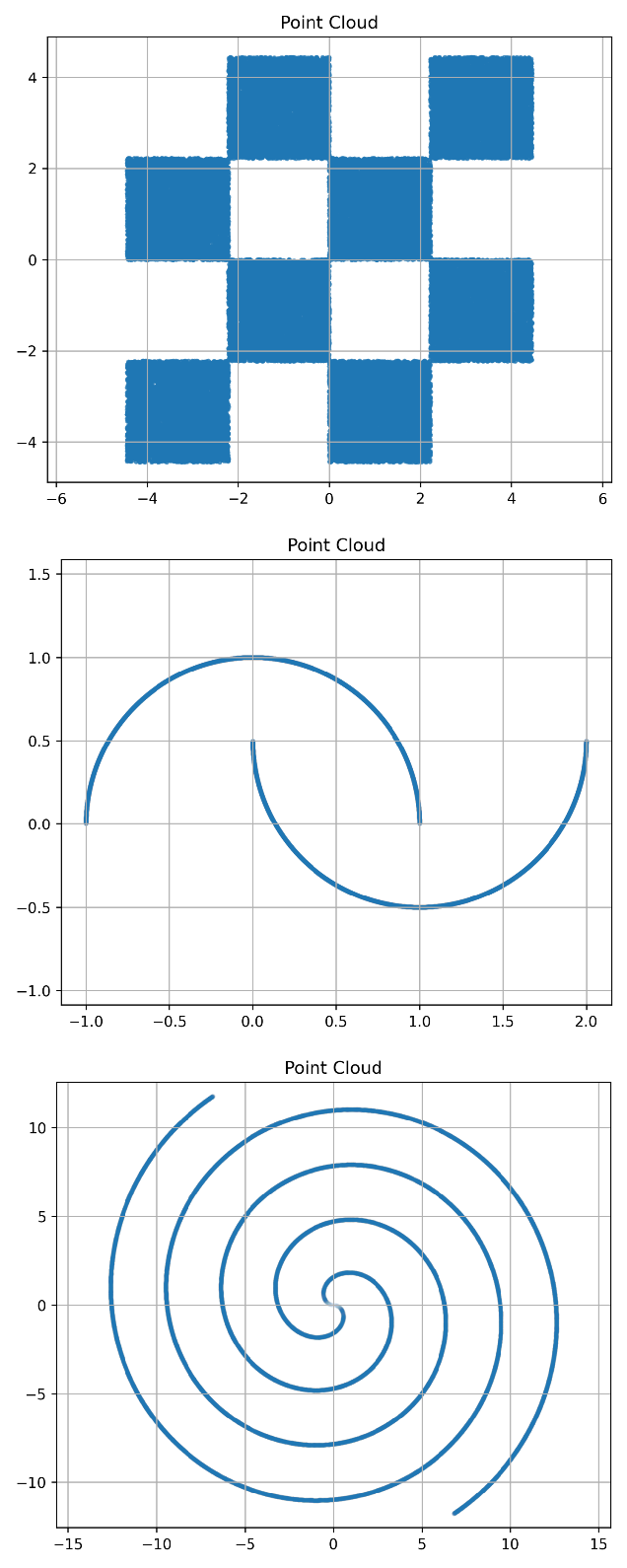}
    \caption{Ground-truth samples}\label{fig:2D_example_1}
  \end{subfigure}\hfill
  \begin{subfigure}[t]{0.24\textwidth}
    \centering
    \includegraphics[width=\linewidth]{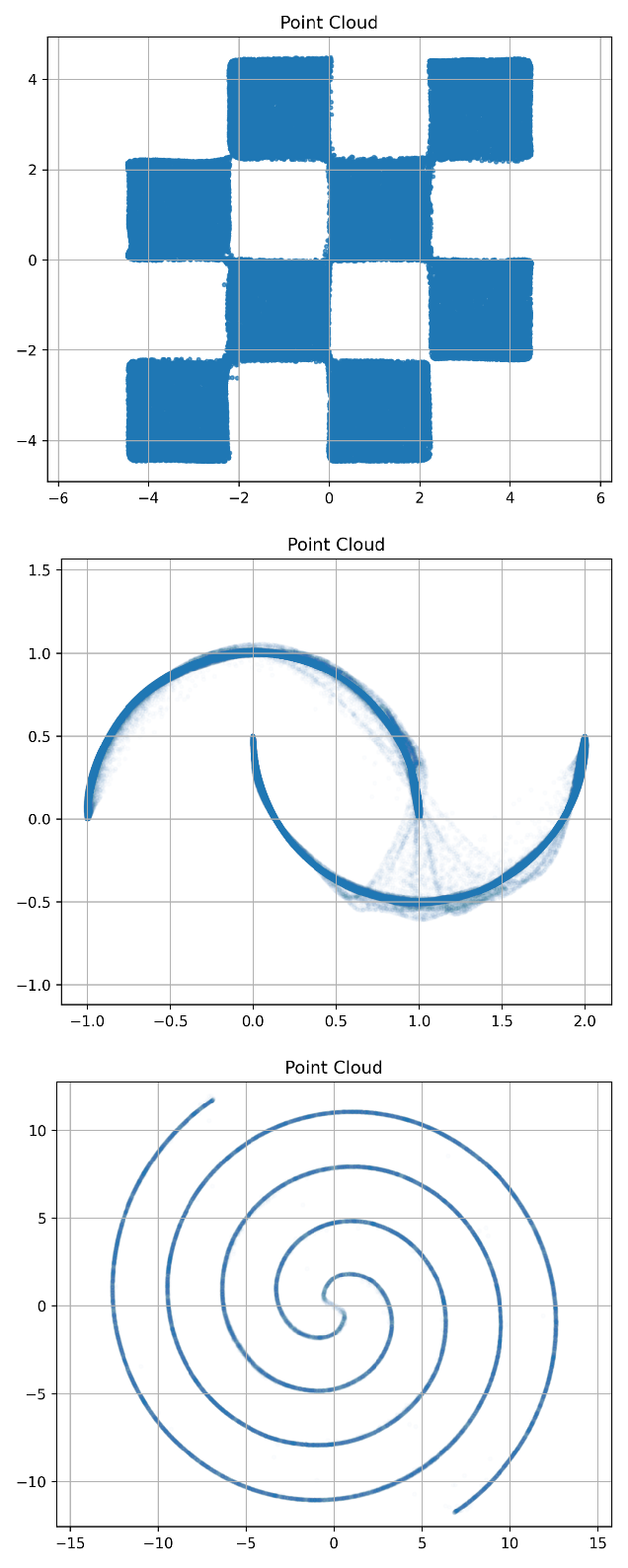}
    \caption{Generated samples}\label{fig:2D_example_2}
  \end{subfigure}\hfill
  \begin{subfigure}[t]{0.24\textwidth}
    \centering
    \includegraphics[width=\linewidth]{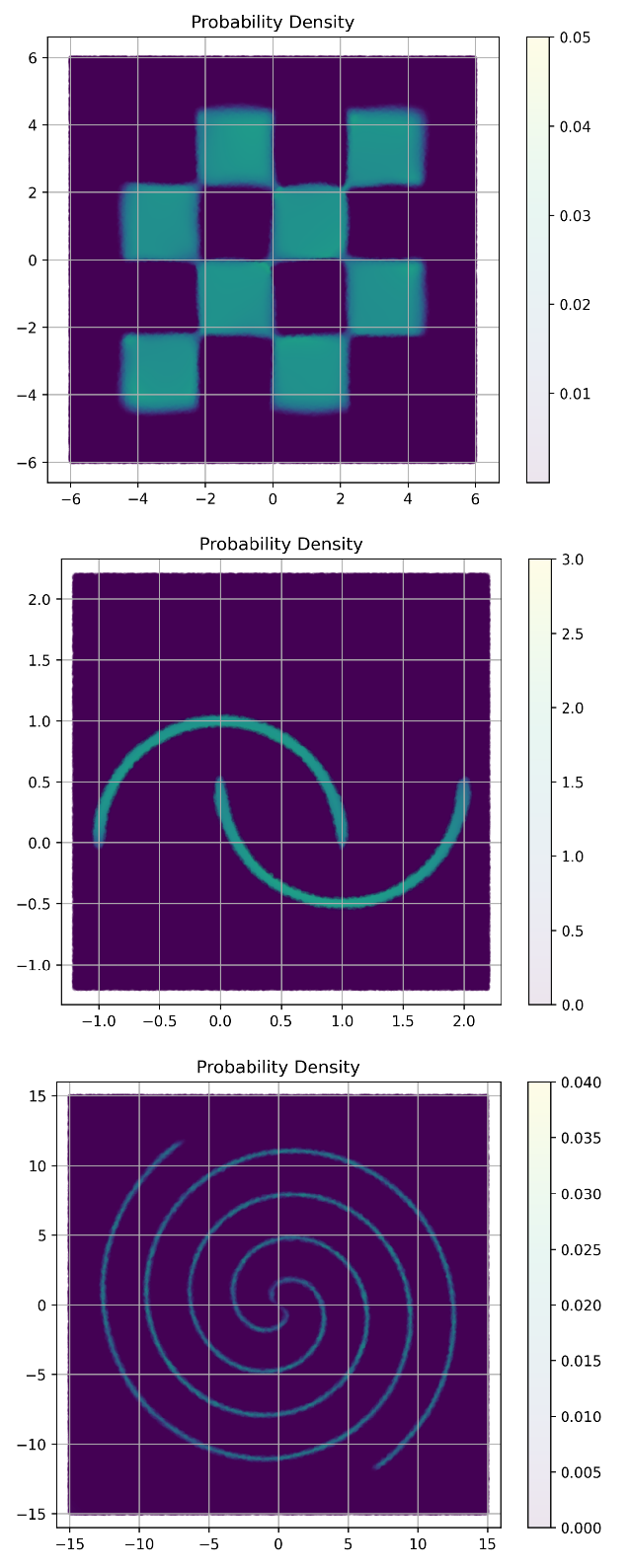}
    \caption{Eulerian likelihood}\label{fig:2D_example_3}
  \end{subfigure}\hfill
  \begin{subfigure}[t]{0.24\textwidth}
    \centering
    \includegraphics[width=\linewidth]{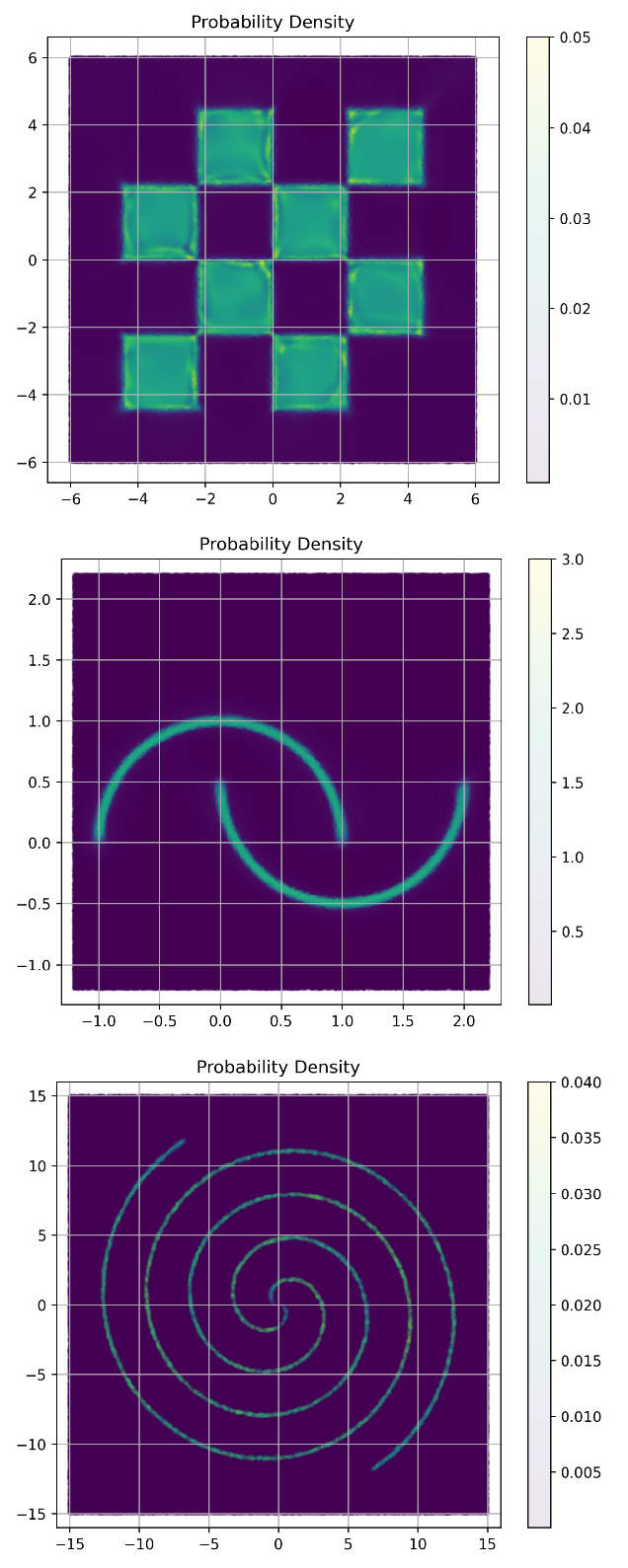}
    \caption{Lagrangian likelihood}\label{fig:2D_example_4}
  \end{subfigure}
  \caption{\textbf{2D Examples.} We show the ground-truth samples, generated samples, and learned likelihoods from the Eulerian and Lagrangian views on Spiral, Checkerboard, and Two Moons.}
  \label{fig:2D_example}
\end{figure}

\subsubsection{Image Generation}\label{sec:image_generation_result}
In \autoref{fig:image_generation1}, \ref{fig:image_generation2}, \ref{fig:image_generation3}, \ref{fig:image_generation4} and \ref{fig:image_generation5}, we evaluate the Long-Short Flow-Map method/the Drifting Model under various learning rates and batch sizes, reporting results at 100K training steps. We consider both the Laplacian kernel used in the original Drifting Model and the Gaussian kernel derived in our analysis, and compute FID \cite{heusel2017gans} by generating 50K images. In \autoref{fig:image_generation6} and \autoref{fig:drifting_ablation}, we conduct ablation studies on feature-space optimization. Specifically, \autoref{fig:image_generation6} reports the quantitative results, while \autoref{fig:drifting_ablation} provides a qualitative comparison of optimization in the default feature setting used in \cite{deng2026generative} (with thousands of feature channels), a reduced feature setting with only a few features, and the original space. In \autoref{fig:bs64_25k}, we show results from the Long-Short Flow-Map method/Drifting Model trained for 400K iterations with batch size 64 (with tuned learning rate), demonstrating that strong sample quality is achievable without extremely large batch sizes. Finally, in \autoref{tab:bs_steps_fid}, we compare our method against other one-step generative baselines.

\begin{figure}[t]
  \centering

  % ---- Row 1: Laplacian kernel ----
  \begin{subfigure}[t]{0.32\textwidth}
    \centering
    \includegraphics[width=\linewidth]{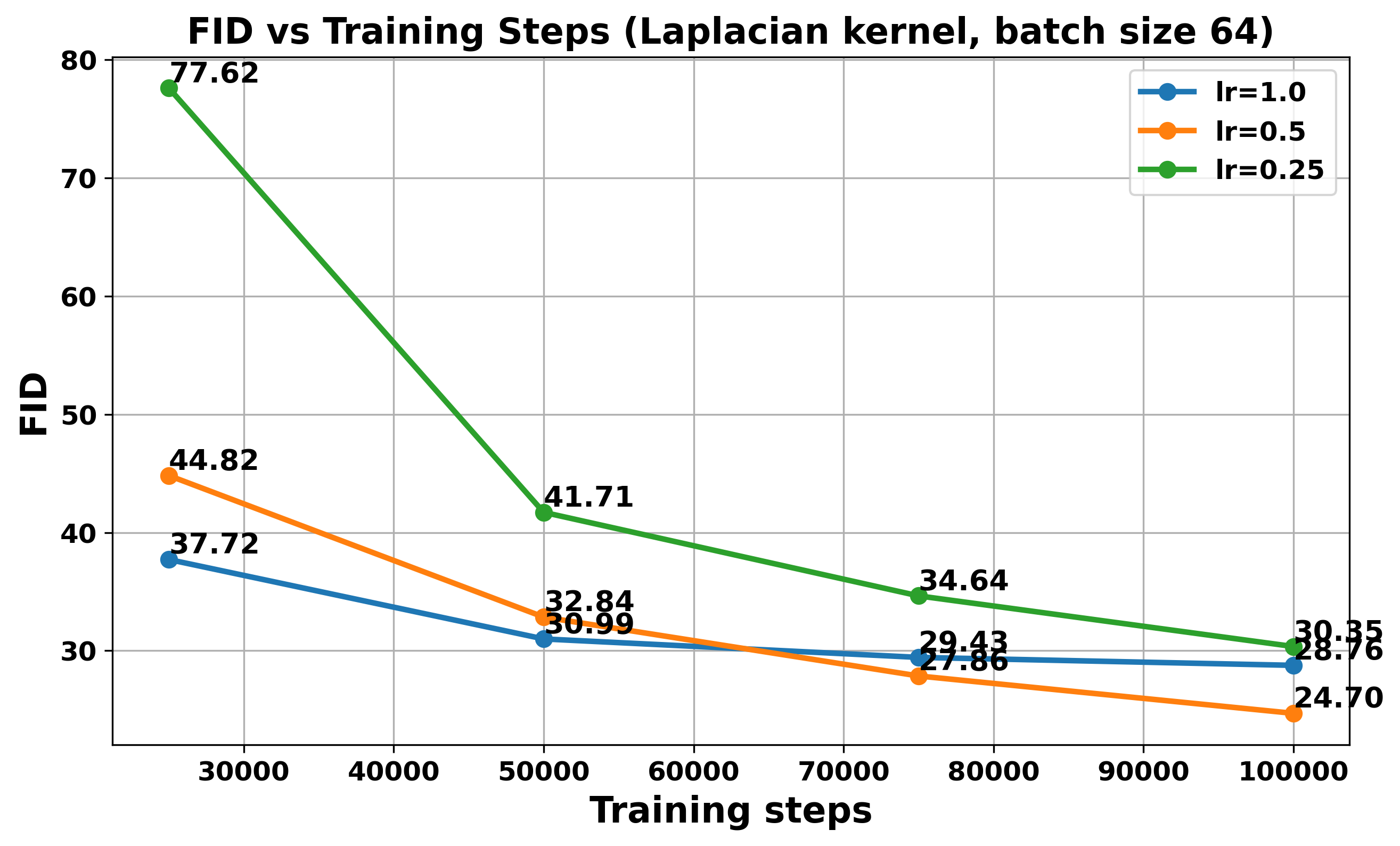}
    \caption{Laplacian kernel, $B{=}64$}\label{fig:image_generation1}
  \end{subfigure}\hfill
  \begin{subfigure}[t]{0.32\textwidth}
    \centering
    \includegraphics[width=\linewidth]{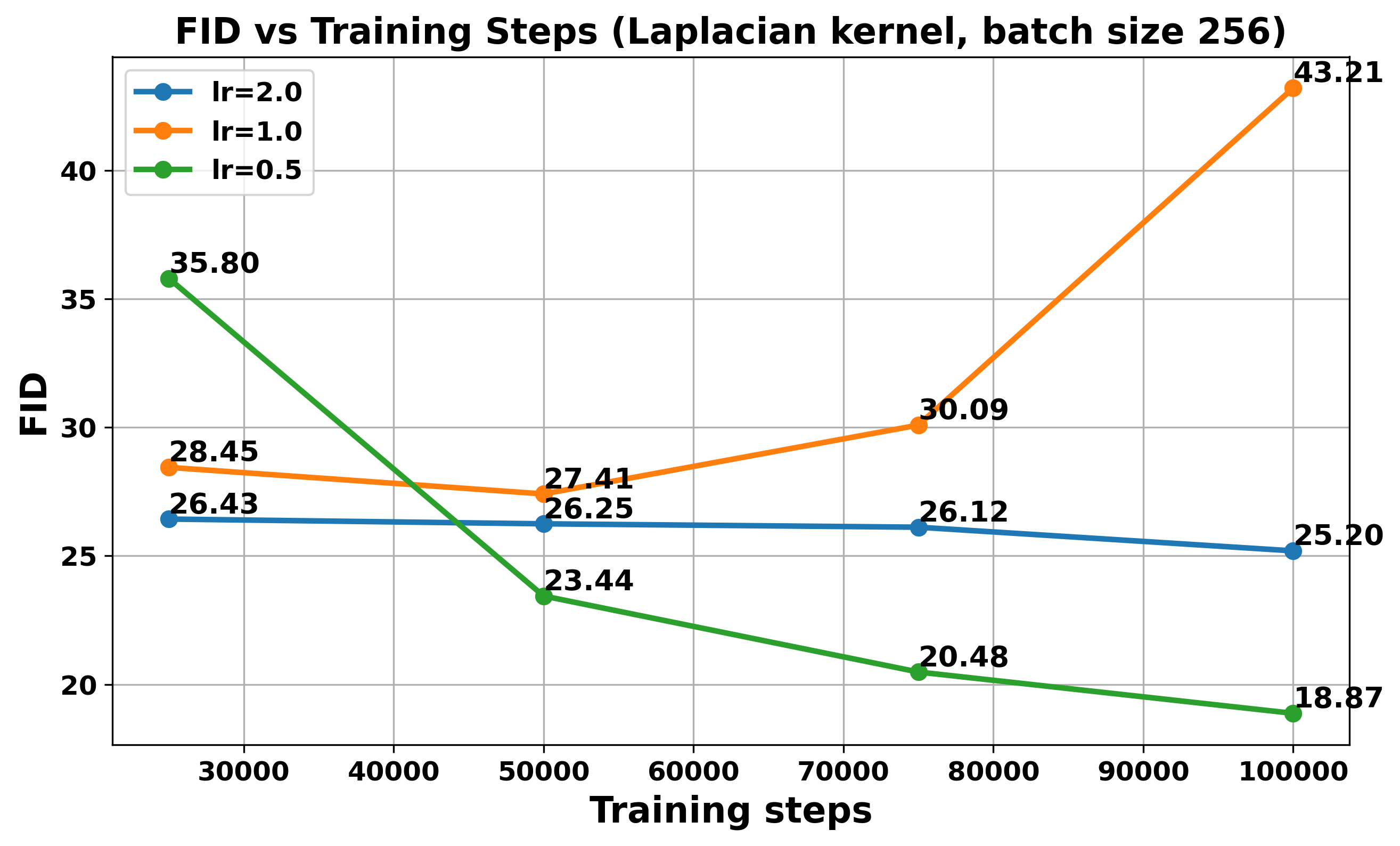}
    \caption{Laplacian kernel, $B{=}256$}\label{fig:image_generation2}
  \end{subfigure}\hfill
  \begin{subfigure}[t]{0.32\textwidth}
    \centering
    \includegraphics[width=\linewidth]{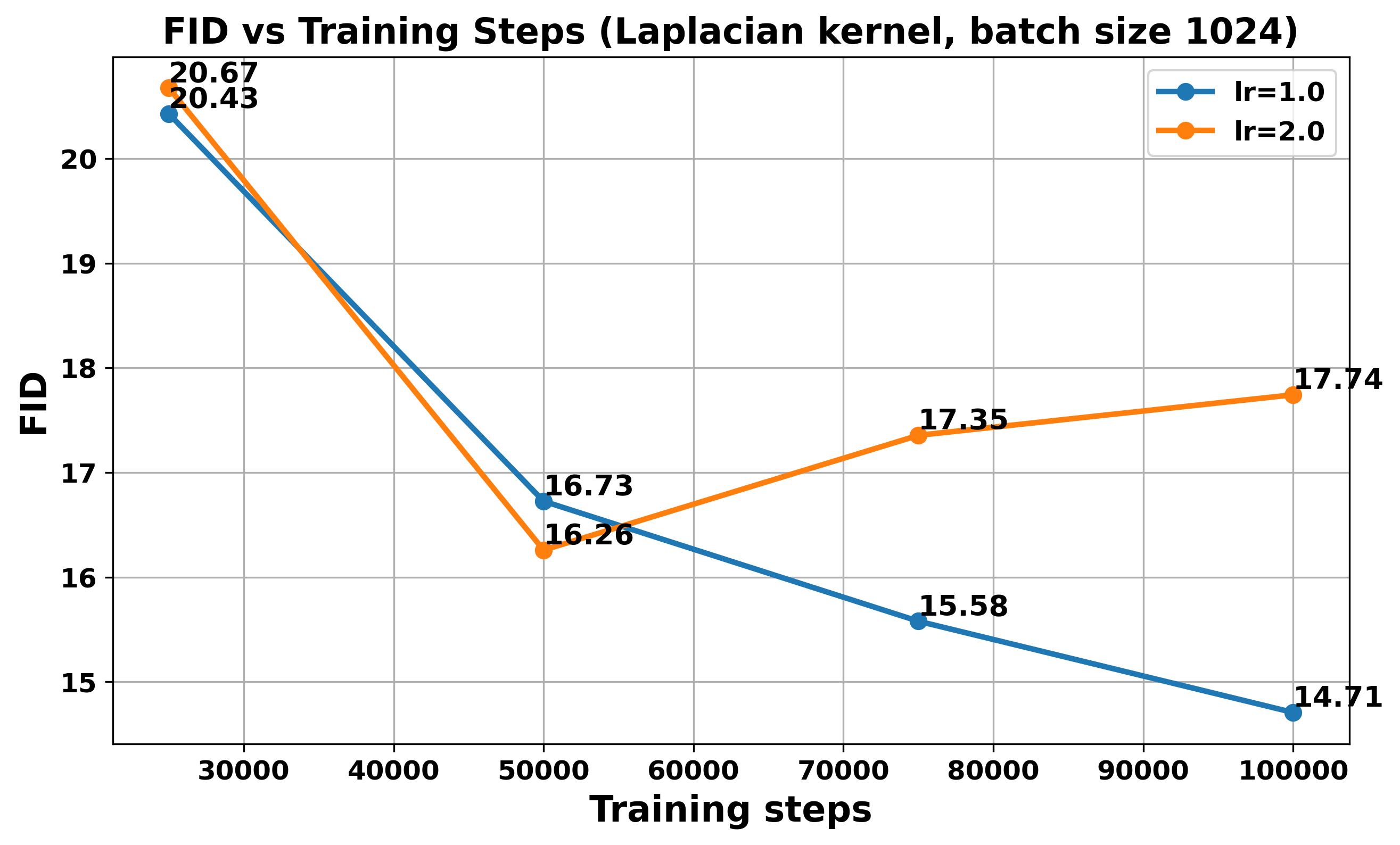}
    \caption{Laplacian kernel, $B{=}1024$}\label{fig:image_generation3}
  \end{subfigure}

  \vspace{0.5em}

  % ---- Row 2: Gaussian kernel + ablations ----
  \begin{subfigure}[t]{0.32\textwidth}
    \centering
    \includegraphics[width=\linewidth]{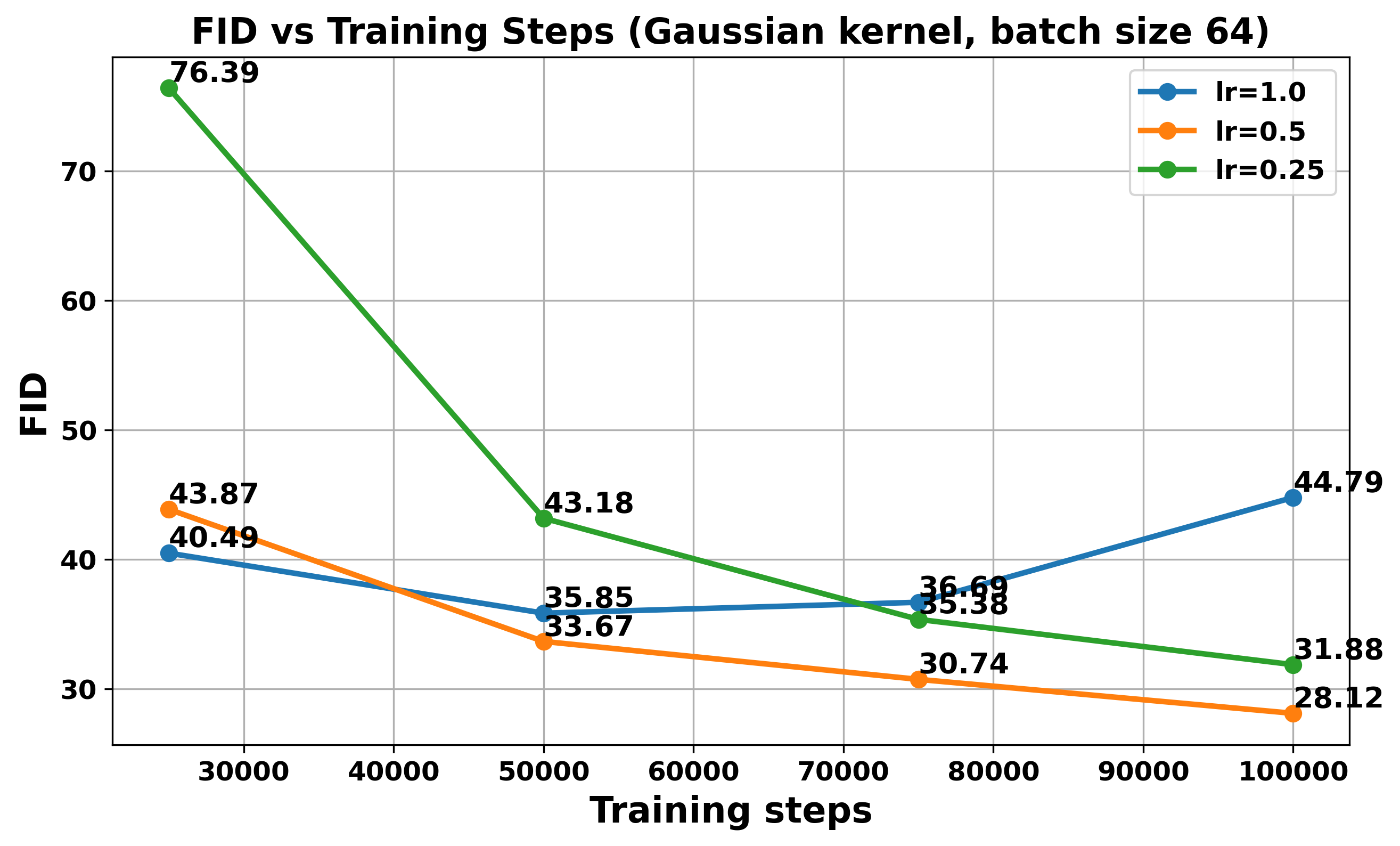}
    \caption{Gaussian kernel, $B{=}64$}\label{fig:image_generation4}
  \end{subfigure}\hfill
  \begin{subfigure}[t]{0.32\textwidth}
    \centering
    \includegraphics[width=\linewidth]{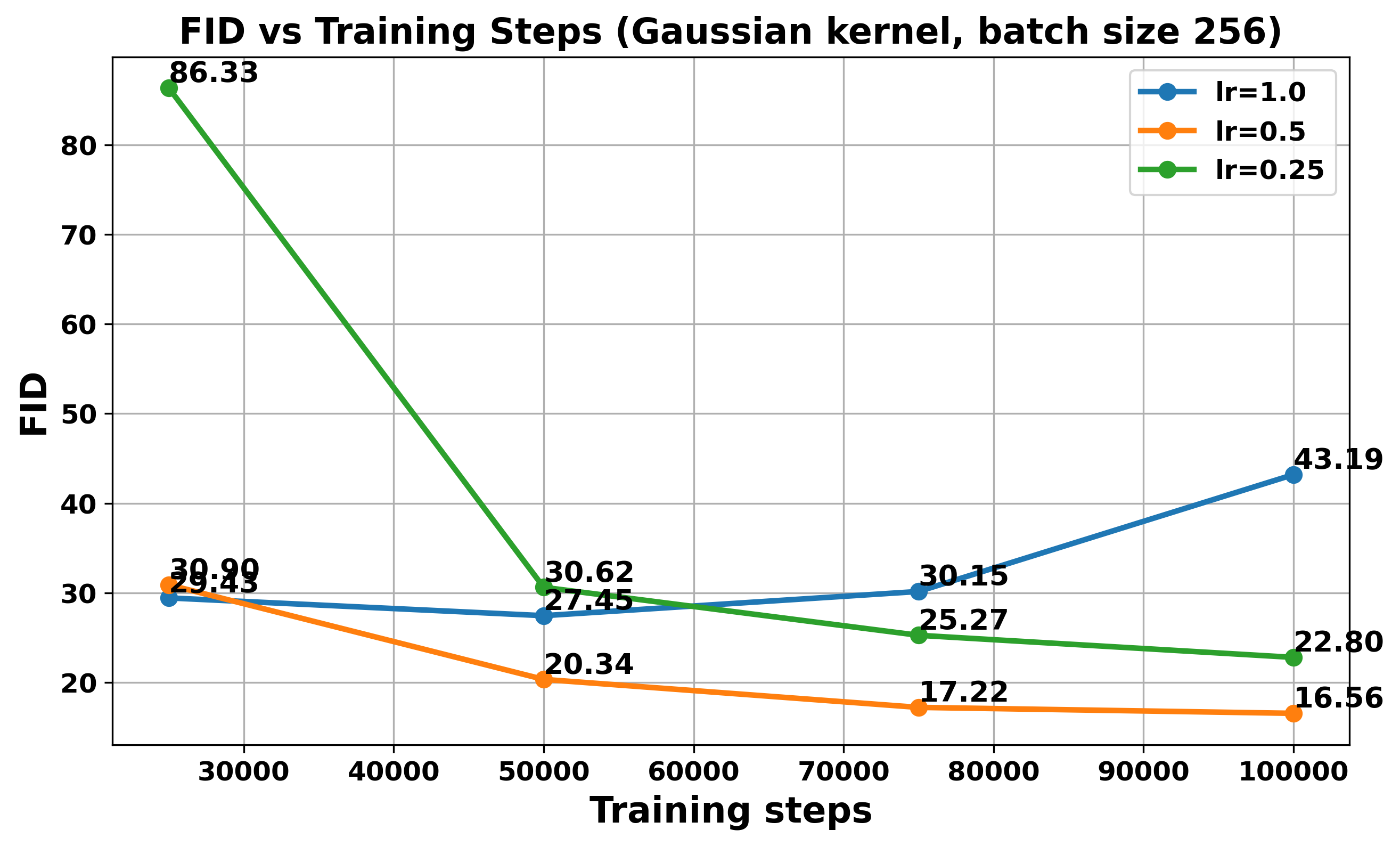}
    \caption{Gaussian kernel, $B{=}256$}\label{fig:image_generation5}
  \end{subfigure}\hfill
  \begin{subfigure}[t]{0.32\textwidth}
    \centering
    \includegraphics[width=\linewidth]{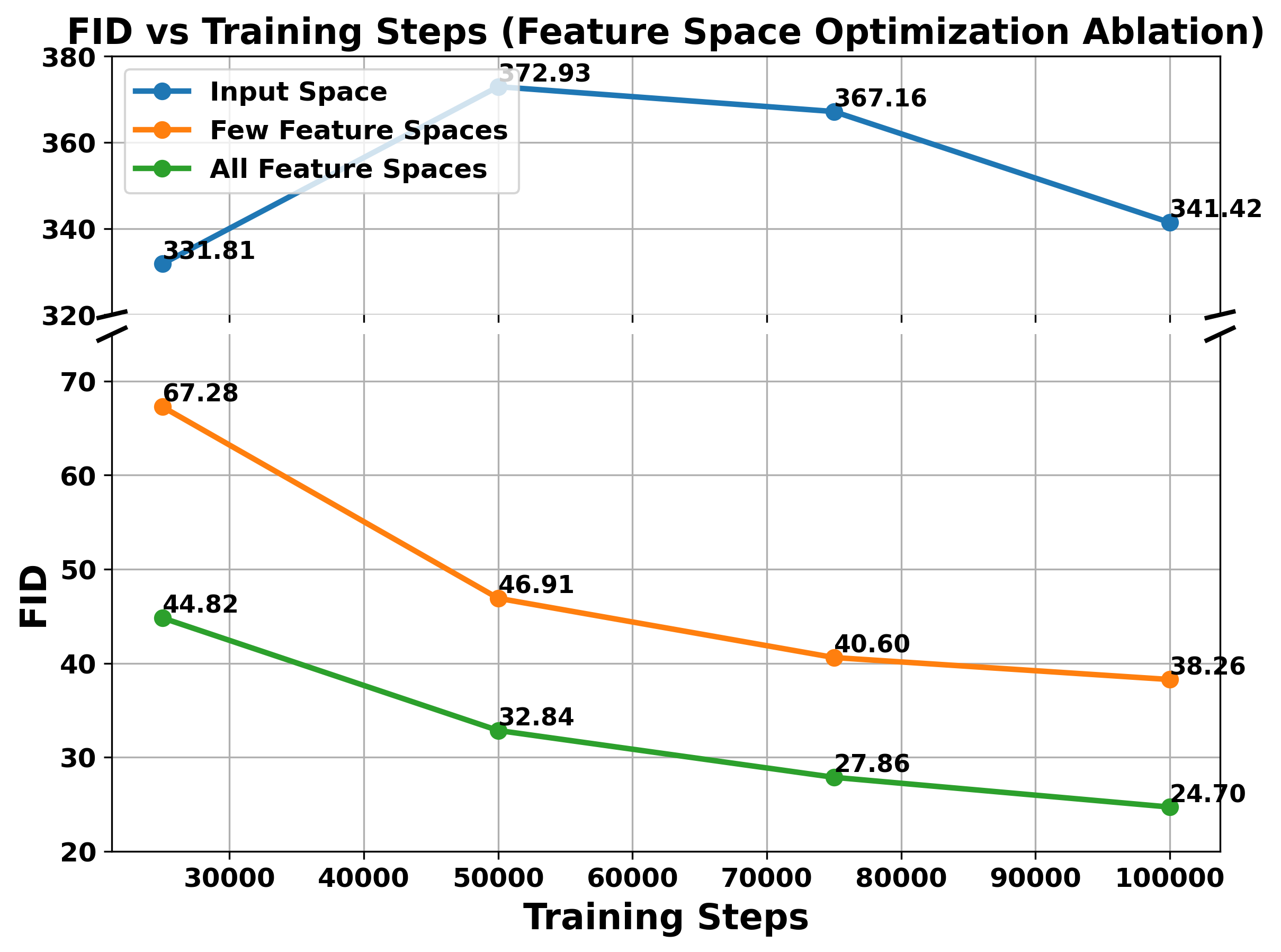}
    \caption{Ablation on multiple feature spaces}\label{fig:image_generation6}
  \end{subfigure}

    \caption{\textbf{Image generation.} FID trajectories under different kernels and batch sizes, together with an ablation on using multiple feature spaces. The Laplacian kernel corresponds to the first-order kernel used in Drifting Model \cite{deng2026generative}, whereas the Gaussian kernel is the second-order kernel derived in our framework. Overall, the Gaussian kernel performs comparably to the Laplacian kernel, and can be better in some regimes. In the ablation (batch size $B{=}64$, lr$ = 5\times 10^{-5}$), we compare using the default feature set with thousands of channels, using only four features (the outputs at the four encoder resolutions), and operating directly in the original space; The latter two settings fail to produce high-quality results.}
  \label{fig:image_generation}
\end{figure}

\begin{figure}[t]
  \centering
  \includegraphics[width=0.33\linewidth]{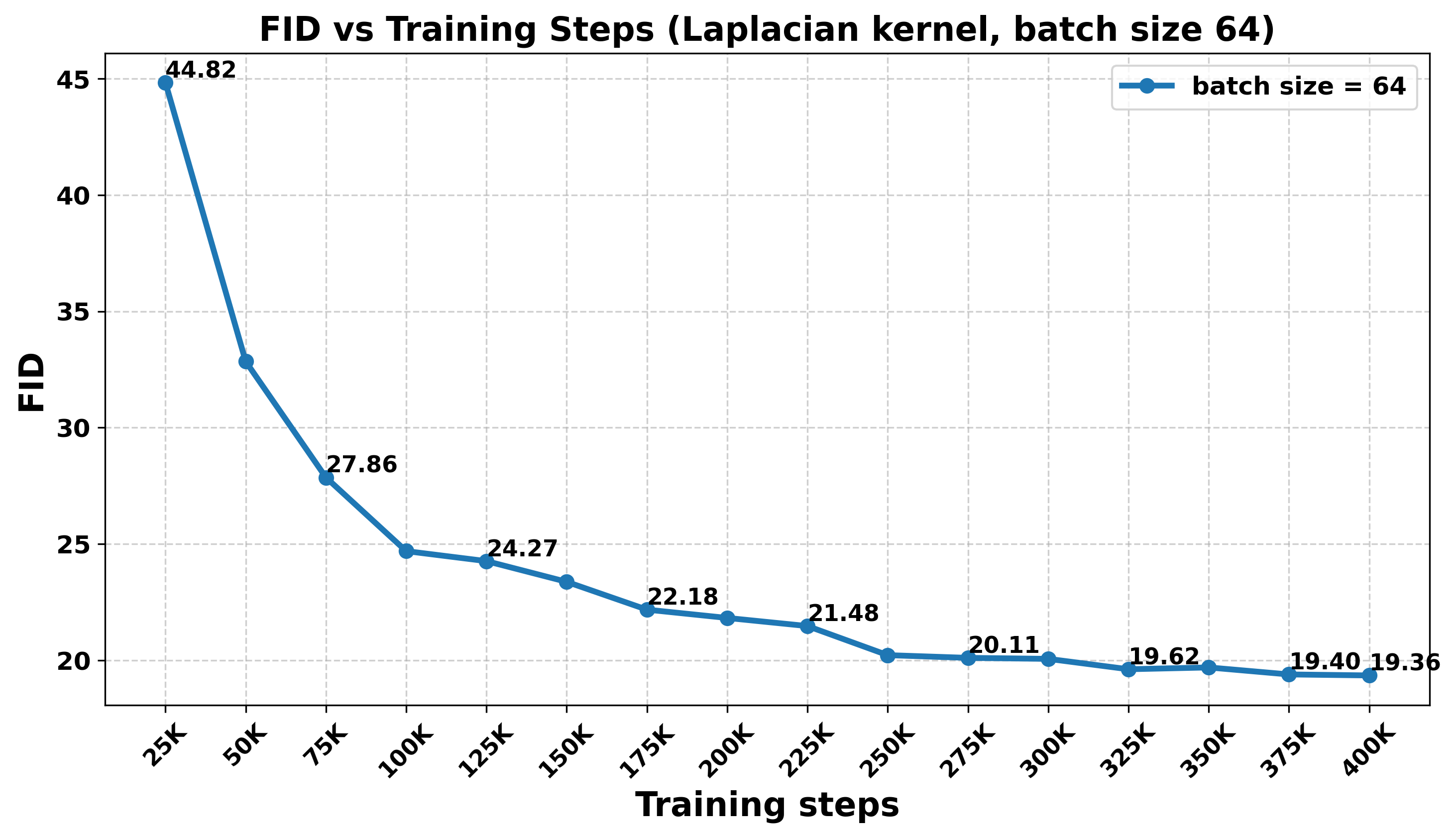}
  \caption{\textbf{FID over training with a standard batch size.} 
  FID curves for the Long-Short Flow-Map method / Drifting Model over training, using batch size 64 (with a tuned learning rate). The results show that strong generative performance can be achieved without extremely large batch sizes.}
  \label{fig:bs64_25k}
\end{figure}

\begin{table}[t]
  \centering
    \caption{\textbf{Comparison under different training budgets.} We report the batch size, training steps, and FID for our method and prior baselines using a DiT-B/2 backbone for one-step latent-space generation.}
  \label{tab:bs_steps_fid}
  \vspace{0.25em}
  \begin{tabular}{lccc}
    \toprule
    \textbf{Method} & \textbf{Batch size} & \textbf{Training steps} & \textbf{FID} \\
    \midrule
    Consistency Models \cite{song2023consistency} & 64 & 400K & 33.2 \\
    Shortcut Models \cite{frans2025shortcut}    & 64 & 400K & 20.5 \\
    MeanFlow \cite{geng2025mean}           & 64 & 400K & 12.4\\

    \midrule
    \multicolumn{4}{l}{\textbf{Long-Short Flow Map / Drifting (ours)}}\\
    \midrule
    \quad Laplacian kernel ($B{=}64$, mixed lr)   & 64 & 400K & 19.36 \\
    \quad Laplacian kernel ($B{=}64$)   & 64 & 100K & 28.12 \\
    \quad Laplacian kernel ($B{=}256$)  & 256 & 100K & 18.87 \\
    \quad Laplacian kernel ($B{=}1024$) & 1024 & 100K & 14.71 \\
    \quad Gaussian kernel  ($B{=}64$)   & 256 & 100K & 24.70 \\
    \quad Gaussian kernel  ($B{=}256$)  & 1024 & 100K & 16.56 \\
    \bottomrule
  \end{tabular}
\end{table}

\begin{figure}[t]
    \centering
    \includegraphics[width=\textwidth]{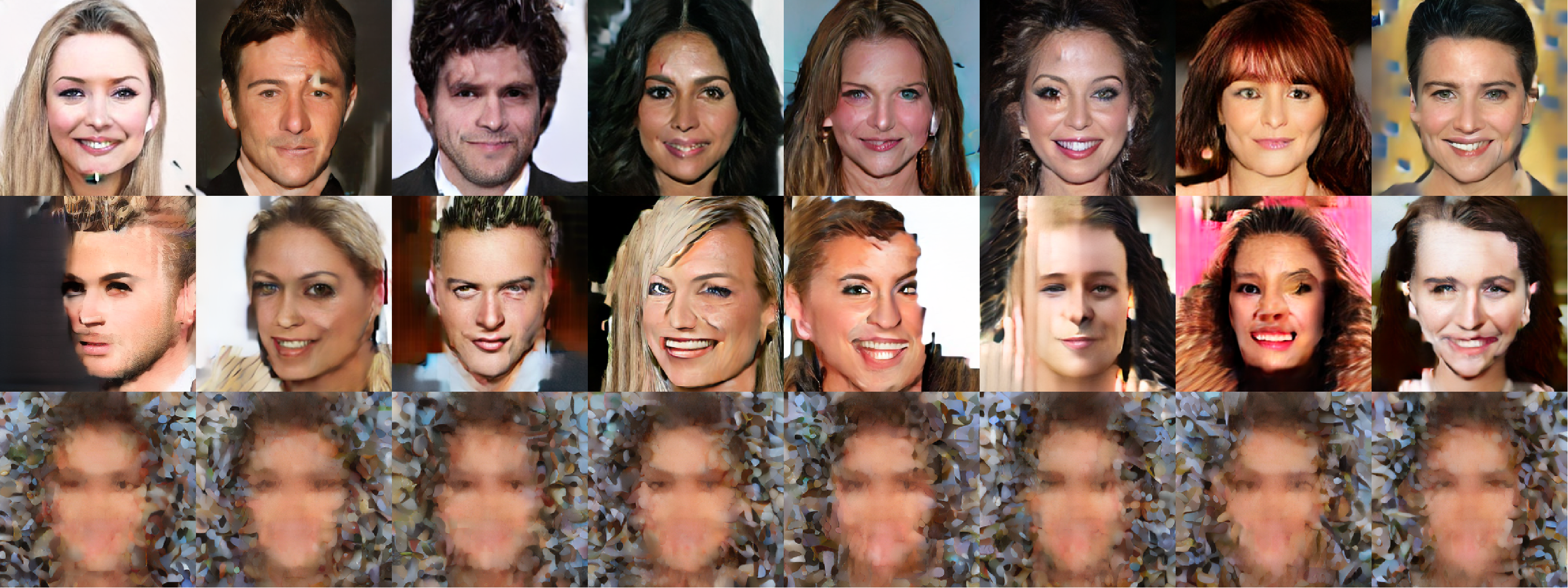}
    \caption{\textbf{Ablation study on feature-space optimization.} We compare three settings: using the default feature set with thousands of features (top), using only four features (middle), and operating directly in the original space (bottom).}
    \label{fig:drifting_ablation}
    \vspace{-1.5ex}
\end{figure}

\section{Discussion on Classifier-Free Guidance}\label{sec:CFG}
Since our computation of the short flow-map segment $\psi_{1-\Delta t\to 1}$ is essentially based on numerically integrating the velocity field, and the velocity is obtained from the closed-form solution in flow matching, it is natural to incorporate CFG when evaluating the velocity on this short segment. Below we outline some possible choices, and exploring a broader design space together with empirical validation is left for future work.
\begin{enumerate}
    \item \textbf{CFG on the velocity field $u_t$.} 
    In \autoref{eq:forward_euler_approx} and \autoref{eq:trapezoidal_rule}, estimating $\psi_{1-\Delta t\to 1}$ introduces the velocity $u_{1-\Delta t}$. With CFG, we can replace the conditional velocity by $u_{1-\Delta t}(x| y)\ \leftarrow\ w\,u_{1-\Delta t}(x| y) + (1-w)\,u_{1-\Delta t}(x| \varnothing)$, where $y$ denotes the class label, $\varnothing$ denotes the null condition, and $w$ is the guidance scale. In the limiting case, this modification turns the $\frac{\mathbb{E}_{x_1\sim p_1}[(x_1-\psi^{\theta}_{0\to 1}(x_0))k_1(x_1,\psi^{\theta}_{0\to 1}(x_0))]}{\mathbb{E}_{x_1\sim p_1}[k_1(x_1,\psi^{\theta}_{0\to 1}(x_0))]}$ term in \autoref{eq:forward_euler_loss} and \autoref{eq:second_order_loss} into $w\frac{\mathbb{E}_{x_1\sim p_1}[(x_1-\psi^{\theta}_{0\to 1}(x_0|y))k_1(x_1,\psi^{\theta}_{0\to 1}(x_0|y))]}{\mathbb{E}_{x_1\sim p_1}[k_1(x_1,\psi^{\theta}_{0\to 1}(x_0|y))]} +(1-w)\frac{\mathbb{E}_{x_1\sim p_1}[(x_1-\psi^{\theta}_{0\to 1}(x_0| \varnothing))k_1(x_1,\psi^{\theta}_{0\to 1}(x_0| \varnothing))]}{\mathbb{E}_{x_1\sim p_1}[k_1(x_1,\psi^{\theta}_{0\to 1}(x_0| \varnothing))]}$.

    \item \textbf{CFG on the target distribution $p_1$.}
    Beyond modifying the velocity field, CFG can also be interpreted as modifying the underlying target distribution \cite{ho2022classifier}. Concretely, one may replace $p_1(x|y)$ by its CFG-guided counterpart $p_1^{\mathrm{cfg}} = p_1^w(x|y)p_1^{1-w}(x|\varnothing)$, which correspondingly transforms the $\frac{\mathbb{E}_{x_1\sim p_1(x_1|y)}[(x_1-\psi^{\theta}_{0\to 1}(x_0))k_1(x_1,\psi^{\theta}_{0\to 1}(x_0))]}{\mathbb{E}_{x_1\sim p_1(x_1|y)}[k_1(x_1,\psi^{\theta}_{0\to 1}(x_0))]}$ term in \autoref{eq:forward_euler_loss} and \autoref{eq:second_order_loss} into $\frac{\mathbb{E}_{x_1\sim p_1^w(x|y)p_1^{1-w}(x|\varnothing)}[(x_1-\psi^{\theta}_{0\to 1}(x_0))k_1(x_1,\psi^{\theta}_{0\to 1}(x_0))]}{\mathbb{E}_{x_1\sim p_1^w(x|y)p_1^{1-w}(x|\varnothing)}[k_1(x_1,\psi^{\theta}_{0\to 1}(x_0))]}$.

    \item \textbf{Other possible directions.}
    The above two options primarily modify the attraction component. Similar to \cite{deng2026generative}, one may also consider applying guidance to the impulse term $\frac{\mathbb{E}_{x^\theta_{1} = \psi_{0\to 1}^\theta(x'_0),x'_0\sim p_0}[(\psi_{0 \to 1}^\theta(x_0)-x_{1}^\theta)k_{1}(\psi_{0 \to 1}^\theta(x_0),x_{1}^\theta)]}{\mathbb{E}_{x^\theta_{1} = \psi_{0\to 1}^\theta(x'_0),x'_0\sim p_0}[k_{1}(\psi_{0 \to 1}^\theta(x_0),x_{1}^\theta)]}$ in \autoref{eq:second_order_loss}.
\end{enumerate}

\end{document}